\documentclass[11pt]{article}

\usepackage{amsfonts,amsmath,amsthm,amssymb}
\usepackage{bm}
\usepackage[colorlinks, citecolor=blue, urlcolor=blue, backref=page]{hyperref}
\usepackage[round]{natbib}
\usepackage{geometry}
\usepackage[utf8]{inputenc}
\usepackage[T1]{fontenc}    
\usepackage{commands}
\usepackage{xcolor} 
\usepackage{xspace}
\usepackage{xifthen}
\usepackage{algorithm}
\usepackage{algorithmic}
\usepackage{commands}
\usepackage{caption}
\usepackage{thmtools}
\usepackage{thm-restate}
\usepackage{graphicx}
\usepackage{wrapfig}

\usepackage[round]{natbib}

\newtheorem{theorem}{Theorem}[section]

\newtheorem{lemma}[theorem]{Lemma}
\newtheorem{corollary}[theorem]{Corollary}

\newtheorem{remark}[theorem]{Remark}

\begin{document}

\title{Logarithmic Neyman Regret\\ for Adaptive Estimation of the Average Treatment Effect}
\author{
Ojash Neopane\footnote{Machine Learning Department, Carnegie Mellon University, \texttt{oneopane@andrew.cmu.edu}} \and
Aaditya Ramdas\footnote{Department of Statistics and Machine Learning Department, Carnegie Mellon University, \texttt{aramdas@cmu.edu}} \and
Aarti Singh\footnote{Machine Learning Department, Carnegie Mellon University, \texttt{aartisingh@cmu.edu}} \\[2ex]
}
\maketitle


\begin{abstract}
    Estimation of the Average Treatment Effect (ATE) is a core problem in causal inference with strong connections to Off-Policy Evaluation in Reinforcement Learning. 
    This paper considers the problem of adaptively selecting the treatment allocation probability in order to improve estimation of the ATE. 
    The majority of prior work on adaptive ATE estimation focus on asymptotic guarantees, and in turn overlooks important practical considerations such as the difficulty of learning the optimal treatment allocation as well as hyper-parameter selection.
    Existing non-asymptotic methods are limited by poor empirical performance and exponential scaling of the Neyman regret with respect to problem parameters. 
    In order to address these gaps, we propose and analyze the Clipped Second Moment Tracking (\clipSMT) algorithm, a variant of an existing algorithm with strong asymptotic optimality guarantees, and provide finite sample bounds on its Neyman regret.
    Our analysis shows that \clipSMT achieves exponential improvements in Neyman regret on two fronts: improving the dependence on $\numRounds$ from $O(\sqrt{\numRounds})$ to $O(\log \numRounds)$, as well as reducing the exponential dependence on problem parameters to a polynomial dependence.
    Finally, we conclude with simulations which show the marked improvement of \clipSMT over existing approaches.
\end{abstract}

\section{Introduction}\label{sec:introduction}

    Randomized Controlled Trials (RCTs) have long been considered the gold standard of evidence in a variety of disciplines, ranging from medicine \citep{Hollis1999Meant}, policy research \citep{Wing2018Designing}, and economics \citep{Banerjee2016Influence}.
    In their simplest form, RCTs involve a control arm and a treatment arm, and the objective is to determine if the treatment \textit{causally} outperforms the control. 
    This is typically achieved by fixing a treatment assignment probability (hereafter called an \textit{allocation}), assigning experimental units to an arm, and using the resulting outcomes to estimate the Average Treatment Effect(ATE). 

    Despite the ubiquity of RCTs, many practitioners have noted that RCTs would benefit from the use of \textit{adaptive} methods---methods in which practitioners vary some aspect of the experiment through the course of the experiment \citep{Chow2005Statistical, Chow2008Adaptive, US2019Adaptive}.
    Although there are many reasons for desiring adaptivity, our primary focus is to adaptively select the treatment allocation probability in order to obtain the best possible estimate of the ATE.
    More concretely, our goal will be to minimize the MSE of our ATE estimate\footnote{In general, one may wish to minimize the mean squared error of the ATE estimate. Since our work focuses on estimation using the unbiased Horvitz-Thompson estimator, this is equivalent to minimizing the variance.}  
    This is the essence of the problem known as Adaptive Neyman Allocation \citep{Dai2024Clip} and is the primary focus of this work.

    Despite the recent attention given to adaptive approaches, considerable work remains to ensure their success in practice.
    This is because a significant portion of prior work on this topic has focused on developing algorithms with strong asymptotic guarantees.
    In this asymptotic regime, much is known, such as the semiparametric efficiency bound \citep{Bickel1993Efficient, Kallus2019DoubleRL} for non-adaptive approaches, as well as adaptive procedures which asymptotically match the performance of the \emph{best possible} non-adaptive approach \citep{Kato2020EfficientAE}. 
    While these results provide a solid foundation, their asymptotic nature overlooks many nuances crucial for practical application. 
    At a high level, prior asymptotic approaches aim to identify the (unknown) variance-minimizing allocation and demonstrate that their allocation converges to this allocation. 
    However, they do not adequately address the challenges of {\em efficiently learning} this allocation, which is often vital for practical implementation \citep{Wagenmaker2023Instance}.

    In order to address these subtleties, we believe a nonasymptotic analysis is required. 
    Unfortunately, such analyses are currently scarce. 
    The only work we are aware of which provides a nonasymptotic analysis is \citet{Dai2024Clip} who propose the \clipOGD algorithm and show it attains $\bigO(\sqrt{\numRounds}$)  Neyman regret---a new measure of performance which we formally introduce in Section~\ref{sec:preliminaries}. 
    Despite offering a promising starting point, this work has several limitations. 
    As we further expand in Sections~\ref{sec:related_works}~and~\ref{sec:algorithm}, \clipOGD can demonstrate poor empirical performance; this is explained by the exponential scaling of their bounds with respect to various problem parameters which they treat as constants.

    In this paper, we advance the understanding of adaptive estimation procedures for the ATE by providing a finite sample analysis of the Clipped Second Moment Tracking algorithm, a variant of the procedure proposed in \cite{Cook2023SemiparametricEI}, tailored for the Horvitz-Thompson estimator. 
    Our analysis meticulously addresses various problem-specific parameters, demonstrating an exponential improvement with respect to problem parameters.
    We also establish a $\bigO(\log \numRounds)$ bound on Neyman regret, representing another significant improvement over \clipOGD, although \cite{Dai2024Clip} consider the more challenging fixed design setting, while we work in the superpopulation setting defined in Section~\ref{sec:preliminaries}.
    Additionally, our finite sample analysis also highlights some aspects of algorithm design that were previously unaddressed.

\section{Related Works}\label{sec:related_works}

Adaptive experimental design has a long and rich history, with work dating back as far as \citet{Neyman1934OnTT}, who introduced the notion of an optimal allocation in the context of experimental designs. Its significance has been emphasized by prominent researchers like \citet{Robbins1952SomeAO}, who identified adaptive sampling as a crucial statistical problem.
Building on this, \citet{Thompson1933ONTL} proposed a Bayesian adaptive design, which later sparked significant interest in the Multi-Armed Bandit (MAB) problem. Thompson’s Bayesian framework introduced the idea of sequentially updating beliefs about the arms (or treatments) based on observed outcomes, which became a foundational concept in MAB research \citep{Lattimore2020BanditA}.
However, the typical MAB formulations studied in the literature often focus on maximizing cumulative rewards across multiple rounds of exploration and exploitation, whereas our approach focuses on a different objective or setting, diverging from the standard MAB treatment.

Our work builds on a recent line of research focused on developing adaptive algorithms to improve the efficiency of ATE estimation.
This line of work begins with \citet{Hahn2009AdaptiveED}, who propose a two-stage design similar to Explore-then-Commit style algorithms \citep{Garivier2016OnES} found in the MAB literature. 
Building on these ideas, \citet{Kato2020EfficientAE} propose a fully adaptive design which we term Regularized Variance Tracking. 
They demonstrate that this approach asymptotically attains the minimum-variance Semiparametric Efficiency (SPE) bound\footnote{By ‘minimum-variance’ SPE bound, we mean minimizing the asymptotic variance jointly over both the estimators and the treatment allocation. The standard SPE bound, in contrast, only considers the minimal asymptotic variance for a fixed treatment allocation.}.
In a follow-up work, \citet{Cook2023SemiparametricEI} propose the Clipped Variance Tracking algorithm, which provides multiple benefits compared to the Regularized Variance Tracking algorithm. 
Specifically, \citet{Cook2023SemiparametricEI} show that their procedure attains the same minimum variance SPE bound under milder assumptions, can be used with modern uncertainty quantification techniques \citep{WaudbySmith2022AnytimevalidOI}, and has superior empirical performance. 
Finally, in an independent line of work, \citet{Li2024OptimalTA} demonstrates that a generalization of the two-stage approach from \cite{Hahn2009AdaptiveED} can be applied to significantly more general problems including Non-Markov Decision Processes.
They show that their approach also obtains the minimum-variance SPE bound.

The main drawback of these prior works is that they only give an asymptotic characterization of their respective procedures, thus painting an incomplete picture of the problem. 
For instance, the asymptotic analysis presented in \citet{Cook2023SemiparametricEI} fails to address the crucial aspect of tuning their algorithm's parameters, an issue that our finite sample analysis clarifies.
These issues in conjunction with the broad asymptotic optimality of `tracking' style algorithms jointly highlight the pressing need for a nonasymptotic analysis.

When compared to asymptotic analyses, research on characterizing the nonasymptotic performance of adaptive algorithms is sparse; we are aware only of the recent work by \citet{Dai2024Clip}, who introduced the \clipOGD algorithm. \citet{Dai2024Clip} analyzes the problem in the \textit{fixed-design} setting and provides bounds on a scaled proxy to the variance of the resulting ATE estimate, termed the Neyman regret.
While this work represents a significant first step, the performance of \clipOGD scales exponentially with respect to problem parameters such as the optimal allocation. 
Additionally, they consider the more challenging fixed-design setting, resulting in pessimistic bounds when assumptions about the underlying data-generating process, such as in the superpopulation setting, are reasonable. 
Our results address these scaling issues and highlight the significant improvements achievable in the superpopulation setting.

The problem of Off-Policy Evaluation (OPE), a generalization of ATE estimation\footnote{OPE is concerned with estimating the performance of a single policy, whereas ATE estimation can be thought of as estimating the difference in performance between two specific policies. However, the techniques developed in the OPE literature can be modified to estimate the difference in performance between two policies}, has also been thoroughly studied within the Reinforcement Learning literature \citep{Dudik2011Doubly, Li2011Unbiased, Jiang2016DoublyRO}. 
Here, the primary focus has been on offline estimation, and there has been extensive work culminating in precise characterizations of minimax lower bounds as well as matching upper bounds \citep{Li2015Toward, Wang2017OptimalAA, Duan2020Minimax, Ma2021MinimaxOE}. 
These ideas have been further generalized to estimating parameters other than the performance of a policy, such as the estimation of the Cumulative Distribution Function of the rewards \citep{Huang2021OffCB, Huang2022OffRL}. 
There is less work on studying adaptive versions of these methods; the only works we are aware of are \citet{Hanna2017DataEfficientPE} who focus on the problem of off-policy learning, and \citet{Konyushova2021Active} who combine offline OPE methods with an online data acquisition strategy to improve the sample efficiency of policy selection, though these works are primarily empirical.

Tangentially related to our work is a line of research on developing inference procedures using adaptively collected data. 
These works can be split into asymptotic and non-asymptotic approaches. 
On the asymptotic side, one line of work has focused on developing new estimators via re-weighting and demonstrating asymptotic normality \citep{Hadad2021Confidence, Zhang2020Inference, Zhang2021Statistical}. 
Another line of work eschews asymptotic results in favor of nonasymptotic results by utilizing modern martingale techniques to develop nonasymptotic confidence intervals and sequences for adaptively collected data, including quantities like the ATE \citep{Howard2018TimeuniformNN, Waudby2023Estimating, WaudbySmith2022AnytimevalidOI}. 

To summarize, while significant progress has been made in the field of adaptive experimental design and related areas, there remain critical gaps, particularly in understanding the non-asymptotic performance of these methods. 
Our work aims to fill these gaps by providing a finite sample analysis that clarifies some aspects of algorithm design and serves as a starting point for analyzing the non-asymptotic behavior of more complicated algorithms.

\section{Preliminaries}\label{sec:preliminaries}

\paragraph{Problem Setup.}
We consider the following interaction between an algorithm, $\Alg$, and a problem instance, $\problemInstance$. 
At the start of each round $\round$, $\Alg$ selects a treatment allocation, $\policy_\round \in [0, 1]$, based on the history of past observations $\history_{\round - 1} = \left\{ \left( \policy_\timeIndex, \action_\timeIndex, \reward_\timeIndex \right) \right\}_{\timeIndex = 1}^{\round - 1}$. 
Then, the next experimental unit is assigned to either the control ($\action_\round = 0$) or the treatment ($\action_\round = 1)$ arm by sampling $\action_\round \sim \bernoulli(\policy_\round)$.
Following this assignment, an outcome $\reward_\round \in [0, 1]$ is observed, marking the end of the round.

We formalize this interaction protocol as follows.
First, we let $\filtration_{\round} = \sigma(\history_\round)$ denote the filtration generated by past observations.
Then an algorithm $\Alg = (\Alg_\round)$ is a sequence of $\filtration_{\round-1}$ measurable mappings, $\Alg_\round : \history_{\round - 1} \rightarrow \probSimplex(\left\{ 0, 1 \right\})$, where $\probSimplex(\placeholderSet)$ is the set of distributions over $\placeholderSet$. 
A problem instance $\nu: \left\{ 0, 1 \right\} \rightarrow \probSimplex([0, 1])$ is a probability kernel which maps each arm to a distribution over outcomes which we assume to be bounded in the interval $[0, 1]$.
Finally, we let $\reward_\round = \indicator[\action_\round = 0] \reward_\round(0) + \indicator[\action_\round = 1] \reward_\round(1)$, where $\indicator[\cdot]$ denotes the indicator function, and $\reward_\round(\action) \sim \problemInstance(\action)$ are called the \textit{potential outcomes}.
Within the causal inference literature, this framework is typically referred to as the superpopulation we  potential outcomes framework \citep{Neyman1923Application, Rubin1980Randomization, Imbens2015Causal}.   

Implicit in the above interaction protocol are the following assumptions:
\begin{enumerate}
    \item \textit{Bounded Observations:} We assume $\reward_\round \in [0, 1]$ almost surely.
    \item \textit{Stable Unit Treatment Value Assumption:} We assume that $\reward_\round(\action)$ is independent of $\reward_\timeIndex(\action)$.
    \item \textit{Unconfoundedness:} Given the history $\history_{\round - 1}$, we assume the treatment assignment $\action_\round$ is independent of the potential outcomes $\reward_\round(0)$ and $\reward_\round(1)$. Formally, $\reward_\round(\action) \perp \action_\round \mid \history_{\round - 1}$ for $\action \in \{0, 1\}$.
\end{enumerate}
While the second and third assumptions are commonplace in the causal inference literature and necessary for identification, the first assumption warrants a brief discussion.
We make this assumption so that our methods are compatible with a recent line of work aimed at developing variance-adaptive sequential hypothesis tests \citep{Karampatziakis2021Off, WaudbySmith2022AnytimevalidOI, Cook2023SemiparametricEI} where it is currently not known how to construct such tests without assuming bounded observations.
However, our analysis and results can be easily modified to accommodate any class of distributions which guarantee concentration of the uncentered second moment.
As we will discuss, this differs from existing work which assumes upper and lower bounds on the raw second moments.
Indeed, our results don't treat any problem parameters as constant, and we take explicit care to understand scaling with respect to all problem parameters.

\paragraph{Efficient Estimation of the ATE.}
This work is concerned with designing algorithms to efficiently estimate the ATE, which we define as $\ATE = \expectation \left[ \reward(1) - \reward(0) \right]$.
Efficient estimation of the ATE is roughly equivalent to minimizing the variance of an estimate of the ATE.
This is because most standard estimates of the ATE are unbiased so the Mean Square Error (MSE) and widths of confidence intervals scale with the variance.
We consider a variant of the standard Horvitz-Thompson (HT) estimator \citep{Imbens2015Causal}.
For a fixed allocation, $\policy$, the HT estimator is defined as
\begin{equation}
    \HT_{\numRounds}(\policy) \definedAs \frac{1}{\numRounds} \sum_{\round = 1}^{\numRounds} \reward_\round \cdot \left( \frac{\indicator\left[ \action_\round = 1 \right]}{\policy} - \frac{\indicator\left[ \action_\round = 0 \right]}{1 - \policy} \right).
\end{equation}
It is well known that the HT estimator is unbiased, and a simple computation shows that its variance is
\begin{equation}\label{eq:fixed-allocation-variance}
\variance\left[ \HT_\numRounds(\policy) \right] = \frac{1}{\numRounds} \left( \frac{\secondMoment{1}}{\policy} + \frac{\secondMoment{0}}{1 - \policy}  - \ATE^{2}\right),
\end{equation}
where $\secondMoment{\action} \definedAs \expectation \left[ \reward^2(\action) \right]$ is the uncentered second moment of $\problemInstance(\action)$. 
The \textit{Neyman allocation}, $\neymanPolicy$, is the allocation that minimizes the variance of this estimate and is given by
\begin{equation}
    \neymanPolicy = \frac{\sqrtSecondMoment{1}}{\sqrtSecondMoment{0} + \sqrtSecondMoment{1}}.
\end{equation}
Although $\neymanPolicy$ depends on unknown parameters, it can be estimated throughout an experiment, as we will do in this work.

\paragraph{Adaptive Estimation of the ATE.}
One issue with the standard HT estimator is that it requires knowledge of the marginal treatment assignment probability $\prob(\action_\round = 1)$. 
For an adaptive algorithm, we do not know this probability \emph{a priori} since it requires marginalizing over the interaction between $\Alg$ and the unknown problem instance $\problemInstance$.
To remedy this issue, \citet{Kato2020EfficientAE} (and later \citet{Dai2024Clip}) proposed the \textit{adaptive Horvitz-Thompson} (aHT) estimator, defined as
\begin{equation}
    \aHT_\numRounds = \frac{1}{\numRounds} \sum_{\round = 1}^{\numRounds} \reward_\round \left( \frac{\indicator\left[ \action_\round = 1 \right]}{\policy_\round} - \frac{\indicator\left[ \action_\round = 0 \right]}{1 - \policy_\round} \right).
\end{equation}
The aHT is also an unbiased estimator of the ATE; this follows from the fact that $\policy_{\round}$ is $\filtration_{\round - 1}$-measurable in conjunction with the law of total expectation. 
In this work, we will aim to design $\Alg$ to minimize the variance of the aHT estimator.

\paragraph{The Neyman Regret.}
The preceding discussion has highlighted that our objective should be to minimize the variance of the aHT estimator.
A straightforward calculation shows that
\begin{equation}\label{eq:adaptive-ht-variance}
    \variance\left[ \aHT_\numRounds \right] = \expectation \left[ \frac{1}{\numRounds^2}\sum_{\round = 1}^{\numRounds} \left( \neymanCost(\policy_\round) - \ATE^{2} \right) \right],
\end{equation}
where 
$$\neymanCost(\policy) = \frac{\secondMoment{1}}{\policy} + \frac{\secondMoment{0}}{1 - \policy}$$ is the \textit{Neyman loss}.
One option to characterize the performance of $\Alg$ is to bound the variance of the resulting ATE estimate.
However, such a quantity fails to normalize against the inherent difficulty of a problem instance --- if the variance of an estimate using the Neyman allocation is large, then a bound on the variance of an adaptive procedure will paint a misleading picture of the performance of $\Alg$.
As such, in this work, we study the \textit{Neyman regret}, $\neymanRegret_\numRounds$, recently introduced in \cite{Dai2024Clip}, and defined as
\begin{equation}
    \neymanRegret_\numRounds = \sum_{\round = 1}^{\numRounds} \neymanCost(\policy_\round) - \neymanCost(\neymanPolicy).
\end{equation}
Inspecting this quantity, we see that the Neyman regret is a scaled difference between the variance of the aHT estimator using an adaptive design and the variance of the HT estimator using the Neyman allocation.
Therefore, demonstrating $\Alg$ obtains sublinear Neyman regret, we know that the variance of the resulting ATE estimate asymptotically approaches the minimum variance estimate.

\paragraph{Notation.} In what follows, we will let $\actionCount{\round}{\action} = \sum_{\timeIndex = 1}^{\round} \indicator\left[ \action_\timeIndex = \action \right]$ denote the number of times the action $\action$ is selected at the end of round $\round$ and $\empiricalSecondMoment{\round}{\action} = \frac{1}{\actionCount{\round}{\action}} \sum_{\timeIndex = 1}^{\round} \reward^{2}_\timeIndex \indicator\left[ \action_\timeIndex = \action \right]$.
We let $\clip(x, a, b) = \min \left( b, \max \left( x, a \right) \right)$.

\section{Algorithm}\label{sec:algorithm}
    In this section, we introduce the Clipped Second Moment Tracking(\clipSMT) algorithm, state bounds on its Neyman regret, and compare its performance with existing algorithms.
    \textit{To simplify our presentation and discussions, in this section, we will assume $\neymanPolicy \leq \frac{1}{2}$.}
    However, we emphasize our results and analysis can be made to hold for all $\neymanPolicy \in (0, 1)$ by flipping the role of the two policies.
    
    \subsection{The \clipSMT Algorithm}
        We begin by describing the \clipSMT algorithm.
        The idea behind this approach is straightforward: since we do not know the Neyman allocation, we instead choose its empirical counterpart,
        \begin{equation}
            \unclippedPolicy_\round = \frac{\empiricalSqrtSecondMoment{\round - 1}{1}}{\empiricalSqrtSecondMoment{\round - 1}{0} + \empiricalSqrtSecondMoment{\round - 1}{1}}.
        \end{equation}
        While this approach is appealing, it will not work without modification.
        This is because $\unclippedPolicy_\round$ is overly sensitive to random fluctuations during the early rounds of interaction.
        As an extreme example, suppose that we select $\action_1 = 1, \action_2 = 0$ and observe $\reward_1 = 0, \reward_2 = 1$. 
        Then, $\unclippedPolicy_\round = 0$ for all proceeding rounds, leading to infinite Neyman regret.
    
        Therefore, we require some form of regularization to guarantee $\clipSMT$ is robust to randomness early in the experiment.
        To regularize $\unclippedPolicy_\round$, we follow \citet{Cook2023SemiparametricEI} and choose the allocation
        \begin{equation}
            \policy_\round = \clip(\unclippedPolicy_\round, \clippingSequence_\round, 1 - \clippingSequence_\round).
        \end{equation}
        for some clipping sequence $\clippingSequence_\round$.
        Our proceeding finite sample analysis will show that setting $\clippingSequence_\round = \frac{1}{2} \round^{-\frac{1}{3}}$ is the correct choice.
        The full algorithm can be found in Algorithm~\ref{alg:clipped-second-moment-tracking}.

        \begin{algorithm}
            \caption{\clipSMT}
            \label{alg:clipped-second-moment-tracking}
            \begin{algorithmic}
                \STATE \textbf{Input:} Clipping sequence $(\clippingSequence_\round)$
                \FOR{each round $t \in \mathbb{N}$}
                    \STATE Compute $\tilde{\pi}_t = \frac{\empiricalSqrtSecondMoment{\round - 1}{1}}{\empiricalSqrtSecondMoment{\round - 1}{0} + \empiricalSqrtSecondMoment{\round - 1}{1}}$
                    \STATE Set $\policy_\round = \clip(\unclippedPolicy_\round, \clippingSequence_\round, 1 - \clippingSequence_\round)$
                    \STATE Play $\action_\round \sim \bernoulli(\policy_\round)$ and observe $\reward_\round$
            \ENDFOR
            \end{algorithmic}
        \end{algorithm}

    \subsection{Understanding the Finite Sample Behavior of \clipSMT}
        We now present our finite sample analysis of \clipSMT.
        To begin, we will assume that the clipping sequence has polynomial decay so that $\clippingSequence_\round = \frac{1}{2} \round^{-\clippingExponent}$ for some $\clippingExponent \in (0, 1)$.
        We discuss alternative choices for $(\clippingSequence_\round)$ in Appendix~\ref{app:clipping-sequences}.
        
        Our analysis splits the behavior of \clipSMT into two phases --- a clipping phase followed by a concentration phase. 
        In the clipping phase, random fluctuations in $\reward_\round$ will induce large variations in $\tilde \policy_\round$, leading our algorithm to clip $\tilde \policy_\round$.
        The clipping phase ends once we can guarantee that our algorithm will no longer clip the plug-in allocation $\unclippedPolicy_\round$, marking the start of the concentration phase, in which we can show that $\policy_\round$ converges to $\neymanPolicy$ at a $\bigO\left( \round^{-\frac{1}{2}} \right)$ rate.
        
        Our first result characterizes the length of the clipping phase for various choices of $\clippingExponent$, demonstrating how to select $\clippingExponent$ appropriately.

        \begin{restatable}{restatedlemma}{clippingphaselengthbound}
        \label{lem:clipping-phase-length-bound}
        Assume for simplicity that $\neymanPolicy \leq \frac{1}{2}$.
            Suppose we run \clipSMT with $\clippingSequence_\round = \frac{1}{2}\round^{-\alpha}$ for $\alpha \in (0, 1)$.
            Let $\worstCaseClippingExponent = \min \left( \clippingExponent, \frac{1 - \clippingExponent}{2} \right)$ and define 
            \begin{equation}
                \clippingPhaseStoppingTime = \tildeBigO \left( \left[ \frac{1}{\neymanPolicy} + \frac{1}{\sqrtSecondMoment{1}}\left( \frac{1}{\sqrtSecondMoment{0}} + \frac{1}{\sqrtSecondMoment{1}} \right)^\frac{1}{2}\log \left( \frac{1}{\errorProb} \right) \right]^{\frac{1}{\worstCaseClippingExponent}}  \right).
            \end{equation}
            Then with probability at least $1 - \errorProb$, for all $\round \geq \clippingPhaseStoppingTime$, we have that $\unclippedPolicy_\round = \policy_\round$.
        \end{restatable}
        
        Before proceeding we make a few remarks about this result.
        First, we can show that there exists a problem instance such that the above bound on the length of the clipping phase is tight (modulo some polylogarithmic factors). 
        This implies that without additional knowledge on $\problemInstance$, setting $\clippingExponent = \frac{1}{3}$ minimizes the length of the clipping phase 
        Furthermore, the proceeding results will show that in the concentration phase $\policy_\round$ converges to $\neymanPolicy$ at a rate that is independent of $\clippingExponent$, thus suggesting that $\clippingExponent = \frac{1}{3}$ is in some sense the correct choice when we don't have additional information about the uncentered second moments.

        The end of the clipping phase indicates sufficient data collection, mitigating the effects of random fluctuations on $\policy_\round$, thus marking the start of the concentration phase. 
        In this phase we can show that $\policy_\round \in \left[ \policy_{\min}, \policy_{\max} \right]$, so that $\actionCount{\round}{1} = \Omega\left( \policy_{\min} \cdot \round \right)$.
        A simple computation shows that this implies $\policy_\round$ converges to $\neymanPolicy$ at a $O\left(\left(  \policy_{\min} \cdot \round\right)^{-\frac{1}{2}} \right)$ rate. 
        While this leads to the correct dependence on $\round$, the scaling with respect to $\policy_{\min}$ is suboptimal--- we expect the scaling to be with respect to $\neymanPolicy$.
        To see why, note that as the interaction progresses, we expect $\policy_\round$ to eventually converge to $\neymanPolicy$. 
        Consequently, we anticipate that $\actionCount{\round}{1} = \Theta\left( \neymanPolicy \cdot \round \right)$ which further implies that $\policy_\round$ converges to $\neymanPolicy$ at a $O\left(\left(  \neymanPolicy \cdot \round\right)^{-\frac{1}{2}} \right)$ rate.
        To remedy this issue, we develop a `double bounding' technique that uses these initial bounds on $\policy_\round$ and refines them to obtain the correct dependence on $\neymanPolicy$.
        This gives us the following result which shows that $\policy_\round$ converges to $\neymanPolicy$ at the desired rate.
        \begin{restatable}{restatedlemma}{concentrationphasepolicybound}
            \label{lem:concentration-phase-policy-bound}
            Assume for simplicity that $\neymanPolicy \leq \frac{1}{2}$.
            Define
            \begin{equation}
                \clippingPhaseStoppingTime = \tildeBigO \left( \left[ \frac{1}{\neymanPolicy} + \frac{1}{\sqrtSecondMoment{1}}\left( \frac{1}{\sqrtSecondMoment{0}} + \frac{1}{\sqrtSecondMoment{1}} \right)^\frac{1}{2}\log \left( \frac{1}{\errorProb} \right) \right]^{3}  \right).
            \end{equation}
            Then with probability at least $1 - \delta$, for all $\round \geq \clippingPhaseStoppingTime$, \clipSMT guarantees that 
            \begin{equation}
                \left\lvert \neymanPolicy - \policy_{\round + 1} \right\rvert \leq \bigO \left( \sqrt{\frac{\ell(\round, \errorProb)}{\round}}  \right)
            \end{equation}
            where $\loglogTerm(\round, \errorProb) = \bigO\left( \log\log \round + \log \frac{1}{\errorProb} \right)$.
        \end{restatable}
        
        The above result shows that following an additional burn-in period after the clipping phase, $\policy_\round$ will converge to $\neymanPolicy$ at the desired $O\left( \left( \neymanPolicy \cdot \round \right)^{-\frac{1}{2}}\right)$ rate.
        We also make a remark about the $\sqrt{\sqrtSecondMoment{\action}}$ terms that appear in our bound.
        These terms appear because of the concentration inequalities we use for $\sqrtSecondMoment{\action}$.
        Unfortunately, we can show that this term is asymptotically unavoidable (see Remark~\ref{rem:second-moment-clt-argument} in Appendix~\ref{sec:concentration}).
    
    \subsection{Bounding the Neyman Regret}
    Before stating our bound on the Neyman regret of \clipSMT, we first give an alternative expression for the simple Neyman regret and provides insight into our Neyman regret bound.
    
    \begin{restatable}{restatedlemma}{neymanregretdecomposition} 
        \label{lem:neyman-regret-decomposition}
        Fix $\policy_\round \in [0, 1]$ and let $\policyDifference_\round = \policy_\round - \neymanPolicy$.
        Then we have that
        \begin{equation}
            \neymanCost(\policy_\round) - \neymanCost(\neymanPolicy) = \Theta\left( \policyDifference_\round^2 \right)
        \end{equation}
    \end{restatable}
    
    The proof of this result can be found in Appendix~\ref{app:clip-smt-analysis}.
    Surprisingly, this result shows that if $\policy_\round$ converges to $\neymanPolicy$ at a $\bigO\left( \round^{-\frac{1}{2}} \right)$ rate, then the simple Neyman regret will shrink at a $\bigO\left( \round^{-1} \right)$ rate.
    Our next result uses this fact in conjunction with the prior bounds on $\policy_\round$ to bound the Neyman regret. 
    \begin{theorem}\label{thm:clip-smt-neyman-regret}
        Assume for simplicity that $\neymanPolicy \leq \frac{1}{2}$. Suppose we run $\clipSMT$ with $\clippingSequence_\round = \frac{1}{2}\round^{-\frac{1}{3}}$. Then probability at least $1 - \errorProb$, the Neyman Regret is at most 
        \begin{equation}
                    \tildeBigO\left( \neymanPolicy^{- 1} \cdot \log(\numRounds) \right).
        \end{equation}
    \end{theorem}

    The proof of this result can be found in Appendix~\ref{app:neyman-regret-bound}. We have just shown that \clipSMT obtains \textit{logarithmic} Neyman regret, providing an exponential improvement from the $O\left( \sqrt{\numRounds} \right)$ Neyman regret obtained by prior works.
    As the proceeding discussion highlights, \clipOGD works in a more general ``design-based'' setup. 
    However, it highlights the significant improvements that can be gained in the superpopluation setting considered in this papers.
    
    \subsection{Comparisons with Prior Work}
        We continue by comparing our results with past works.
        
        \paragraph{Comparison with \citet{Dai2024Clip}.} 
        When comparing our Neyman regret bounds to \clipOGD, we observe exponential improvements in scaling with respect to $\neymanPolicy$ and $\numRounds$.
        
        Starting out with the dependence on $\neymanPolicy$, our bound scales like $\bigO\left( \neymanPolicy^{-1} \right)$ while \clipOGD scales like $\bigO \left( \exp\left( \neymanPolicy^{-1} \right) \right)$.
        We remark that it is not fully clear if the exponential scaling for \clipOGD is a product of the proof technique or is a fundamental drawback of \clipOGD.
        Inspecting the proof in \citet{Dai2024Clip}, this exponential dependence is introduced to tune the learning rate---if bounds on $\neymanPolicy$ are known, \clipOGD can be tuned to scale polynomially in $\neymanPolicy^{-1}$.
        However, even then, not only is the exponent in their polynomial always worse than ours, but it also scales with  $\sqrt{\log \numRounds}$, while \clipSMT does not.
        Finally, we empirically observe that \clipOGD is sensitive to parameter choices.
        The choices suggested by their analysis can often lead to poor performance (as we demonstrate in Section~\ref{sec:experiments}) indicating that the aforementioned exponential dependence is indeed a fundamental drawback.

        Next, we see that our Neyman regret scales like $\bigO(\log \numRounds)$ while \clipOGD scales like $\bigO(\sqrt{\numRounds})$.
        While this is an exponential improvement, we believe this difference is primarily due to the differences in our problem settings---we consider the superpopulation setting where outcomes are stochastic whereas \citet{Dai2024Clip} consider the fixed-design setting where the outcomes are a fixed sequence.
        In the fixed-design setting, the potential outcomes can be chosen adversarially, including with knowledge of the algorithm, thus increasing the problem's difficulty.
        The differences between these settings parallels the differences between stochastic and adversarial MABs where we observe similar gaps in regret bounds.
        In the stochastic bandit setting, the best one can obtain is $\bigO \left( \log \numRounds  \right)$ problem dependent regret\citep{Auer2002Finite}; whereas in the adversarial bandit setting, the best one can hope to do is $\bigO\left( \sqrt{\numRounds} \right)$ minimax regret \citep{Auer1995Gambling}.
        
        \paragraph{Comparison with \cite{Cook2023SemiparametricEI}.}
        As we have mentioned, our algorithm is a variant of the algorithm proposed by \citet{Cook2023SemiparametricEI}, tailored to the aHT estimator.
        The primary difference between our work and \citet{Cook2023SemiparametricEI} is that their focus is asymptotic while ours is nonasymptotic.
        The asymptotic perspective makes design choices such as the appropriate clipping sequence opaque.
        In their concluding remarks \citet{Cook2023SemiparametricEI} state that selection of the clipping sequence is an interesting question for future work -- our finite sample analysis gives a concrete answer to this question.
        As an example of the difficulty in choosing the clipping sequence, \citet{Cook2023SemiparametricEI} uses a clipping sequence with exponential decay.
        Our finite sample analysis indicates that with constant probability, such a clipping sequence will result in an allocation that does not converge to $\neymanPolicy$.
        Finally, we remark that using a clipping sequence with polynomial decay allows us to slightly generalize their asymptotic results by removing the requirement that bounds on $\neymanPolicy$ are known.

\section{Experiments}\label{sec:experiments}

 \begin{figure}[ht!]
    \centering
    \includegraphics[width=\columnwidth]{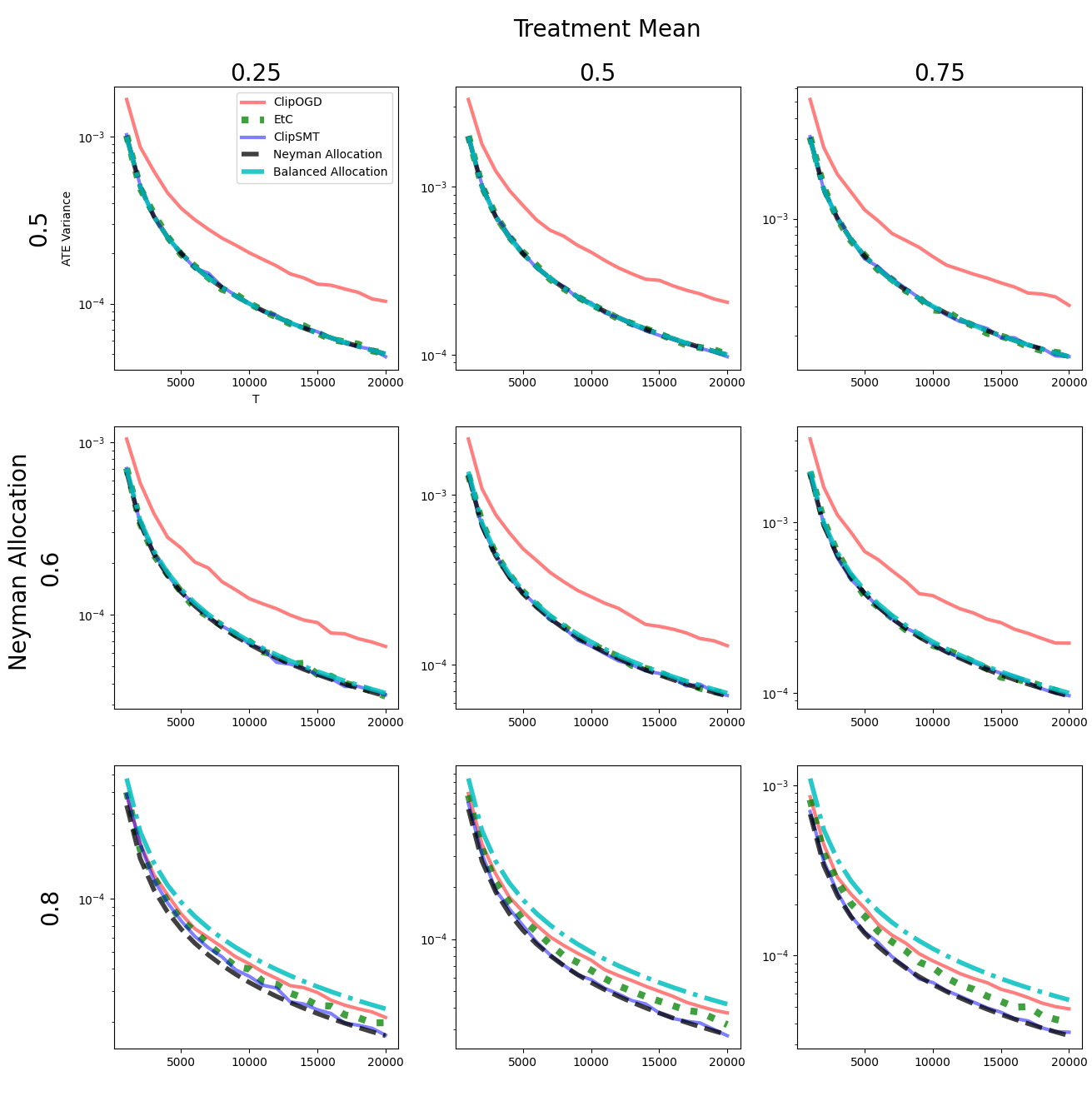}
    \captionsetup{margin=0.5cm}
    \caption{Comparison of the performance of \clipSMT, \clipOGD, Explore-then-Commit (\EtC), Neyman allocation, and a balanced allocation with the treatment and control arms following Bernoulli distributions. Individual subplots plot the variance of each design against the number of samples for a fixed problem instance. Each column keeps the treatment mean fixed, and each row keeps the Neyman allocation fixed. Moving to the right increases the treatment mean and moving down increases the Neyman allocation. Overall the performance of \clipSMT is always competitive with the performance of the infeasible Neyman allocation and outperforms the other adaptive designs. Furthermore, as the Neyman allocation increases, we see that \clipSMT adapts to the increased difficulty while \EtC and the balanced design do not.}
    \label{fig:alg-comparisons}
\end{figure}

In this section, we experimentally validate our algorithm. Our objectives are to compare our algorithm to existing approaches and sensible baselines and to understand how well our theoretical characterization of \clipSMT aligns with its empirical behavior.

We start by comparing our algorithm to existing approaches and some non-adaptive baselines. In these experiments, we compare \clipSMT with \clipOGD, the infeasible Neyman Allocation, a balanced allocation with $\policy = \frac{1}{2}$, and a two-stage design we call Explore-then-Commit (\EtC). For \EtC, we select each treatment arm with equal probability for $\numRounds^{\frac{1}{3}}$ rounds, after which we compute the empirical Neyman allocation and use this allocation for the remaining rounds.

We evaluate each approach on nine problem instances, running them for $\numRounds$ rounds, where $\numRounds$ varies from 1000 to 20000 in increments of 1000. For each fixed value of $\numRounds$, we run \clipSMT, \clipOGD, and the two-stage design 5000 times to approximate the variance of the resulting ATE estimate. For the Neyman and balanced allocations, we compute their variance using equation~\eqref{eq:fixed-allocation-variance}. Our results show that \clipSMT outperforms \clipOGD and \EtC, and adapts well to difficult problem instances (i.e., when the Neyman allocation deviates from $\frac{1}{2}$). The results of the experiments are displayed in Figure~\ref{fig:alg-comparisons}.

Next, we validate whether the length of the clipping phase predicted by our theory aligns with the empirical behavior of \clipSMT.
To do this, we run \clipSMT using $\clippingSequence_\round = \frac{1}{2} \round^{-\clippingExponent}$ for various values of $\clippingExponent \in (0, \frac{1}{2})$.
For each problem instance and each value of $\clippingExponent$, we compute the end of the clipping phase.
We perform this experiment 5000 times and use these trials to determine the 0.95 quantile of when the clipping phase ends.
Using these values, we compute the ratio of the clipping time predicted by our theory to the empirically computed clipping time.
The results of this experiment are shown in Figure~\ref{fig:clipping-ratio}.

\begin{figure}[ht!]
    \centering
    \includegraphics[width=\columnwidth]{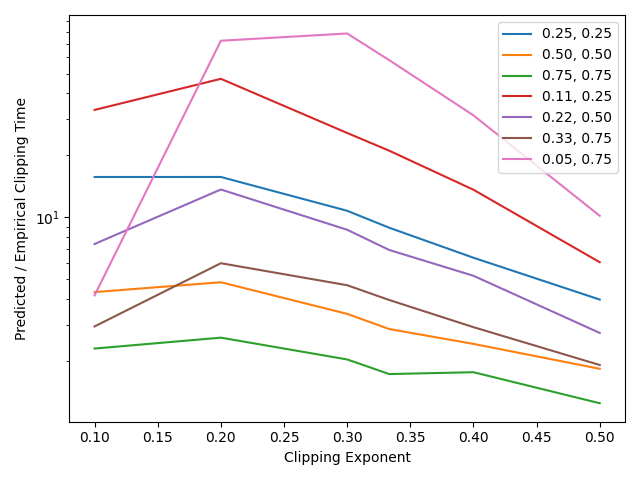}
    \captionsetup{margin=0.5cm}
    \caption{Plot of the ratio between the clipping time predicted by Lemma~\ref{lem:clipping-phase-length-bound} and the empirically computed clipping time against different values for $\clippingExponent$. Each line represents the ratios for a different problem instance. We see that for most problem instances the ratio remains relatively constant for moderate values of $\clippingExponent$.}
    \label{fig:clipping-ratio}
\end{figure}

Upon inspecting these results, we find that for most problem instances, this ratio remains roughly constant when $\clippingExponent \in (0.1, 0.4)$.
However, when the Neyman allocation is large, the predictions made by our theory do not accurately predict the clipping time.
This discrepancy arises from a looseness in our proof.
Specifically, to bound the clipping time, we need to bound a quantity of the form $\min \left\{ \round : \sum_{i} \round^{p_i} \geq c_1 + c_2 \log\log \round \right\}$.
Instead, we bound $\min \left\{ \round : \round^{\max p_i} \geq c_1 + c_2 \log\log \round \right\}$, which provides the correct bound for large values of $c_1$ and $c_2$ but is loose for moderate values of $c_1$ and $c_2$.
Therefore, while our bound is tight in the worst case, it has some looseness for specific problems—resolving this issue remains an interesting technical problem for future work.

\section{Conclusion}\label{sec:conclusion}
In this paper, we performed a finite sample analysis of the \clipSMT algorithm for adaptive estimation of the ATE.
Our analysis clarified several aspects of algorithm design, including how to properly tune the clipping sequence.
Furthermore, we demonstrated that our approach achieves exponential improvements in two distinct areas when compared to the only other method with a finite time analysis. 
Our comprehensive analysis meticulously addressed all problem parameters, providing a clearer and more detailed understanding of the complexity of adaptive ATE estimation.

Several promising directions for future work emerge from our findings. 
One obvious direction is to extend our analysis to the Augmented Inverse Probability Weighted estimator, which has more desirable properties and is more appropriate when contextual information is available.  
Additionally, expanding these results to accommodate larger action spaces and stochastic context-dependent policies warrants further discussion.

\bibliography{ref}
\appendix
\newpage
\appendix
\onecolumn

\section{Analysis of ClipSMT}\label{app:clip-smt-analysis}
    In this section, we will prove our bound on the Neyman regret of \clipSMT.
    \subsection{Preliminaries.} 
        Before we proceed to the analysis, we first introduce some notation and define a `good event' which we will assume to hold throughout the analysis. 
        We define the following events
        \begin{align}
            \goodEventCount(\countErrorProb) &= \bigcap_{\round = 1}^{\infty} \left\{\actionCount{\round}{1} \in \left[\sum_{\timeIndex = 1}^{\round} \policy_\timeIndex - \countConfidenceWidth(\round, \countErrorProb), \sum_{\timeIndex = 1}^{\round} \policy_\timeIndex + \countConfidenceWidth(\round, \countErrorProb)  \right] \right\} \\
            \goodEventSecondMoment(\sqrtSecondMomentErrorProb) &= \bigcap_{\action \in \left\{ 0,1 \right\}} \bigcap_{\round = 1}^{\infty} \left\{\sqrtSecondMoment{\action} \in \left[\empiricalSqrtSecondMoment{\round}{\action} - \sqrtSecondMomentConfidenceWidth(\round, \errorProb), \empiricalSqrtSecondMoment{\round}{\action} + \sqrtSecondMomentConfidenceWidth(\round, \sqrtSecondMomentErrorProb) \right]\right\}.
        \end{align}
        Applying Lemmas~\ref{lem:concentration-bernoulli-counts}~and~\ref{lem:sqrt-second-moment-concentration},  using $\countConfidenceWidth$ and $\sqrtSecondMomentConfidenceWidth$ respectively defined in equations~\eqref{eq:count-confidence-width}~and~\eqref{eq:second-moment-confidence-width} with $\countErrorProb = \frac{\errorProb}{3}, \sqrtSecondMomentErrorProb = \frac{2 \delta}{3}$, we see that the event $\goodEvent = \goodEventCount(\countErrorProb) \bigcap \goodEventSecondMoment(\sqrtSecondMomentErrorProb)$ occurs with probability at least $1 - \errorProb$.
        For the remainder of the section, we will assume that this event hold.
        

    
    \subsection{Bounding the Neyman Regret (Theorem~\ref{thm:clip-smt-neyman-regret})}\label{app:neyman-regret-bound}

        We will bound the cumulative Neyman regret by bounding the simple Neyman regret and then summing over those terms. In order to do so, we will handle the clipping phase and concentration phases separately.

        For the clipping phase, Lemma~\ref{lem:bounds-on-policy} demonstrates that we can guarantee $\policy_\round \in [\policyMin, \policyMax]$ where $\policyMin, \policyMax$ only depend on $\sqrtSecondMoment{\action}$.
        This implies that the instantaneous Neyman regret for each round in the clipping phase can be upper bounded by a constant $\clippingPhaseSimpleNeymanRegret = \max_{\policy \in \left\{ \policyMin, \policyMax \right\}} \neymanCost(\policy) - \neymanCost(\neymanPolicy)$ which only depends on $\sqrtSecondMoment{\action}$.
        Furthermore, Lemma~\ref{lem:clipping-phase-length-bound} shows that the length of the clipping phase is at most $\clippingPhaseStoppingTime$ so that the cumulative Neyman regret from the clipping phase can be upper bounded as $\clippingPhaseSimpleNeymanRegret \cdot \clippingPhaseStoppingTime$ which is independent of $\numRounds$.

        For the concentration phase, we apply Lemma~\ref{lem:concentration-phase-policy-bound} which shows that $\policyDifference_\round \leq \tildeBigO\left( \round^{-\frac{1}{2}} \right)$ so that Lemma~\ref{lem:neyman-regret-decomposition} implies that the instantaneous Neyman regret for each round of the concentration phase is at most
        \begin{equation}
            16\left( \frac{1}{\sqrtSecondMoment{0} + \sqrtSecondMoment{1}} \right)^2 \left( \frac{1}{\sqrt{\sqrtSecondMoment{0} \left( 1 - \neymanPolicy \right)}} + \frac{1}{\sqrt{\sqrtSecondMoment{1} \neymanPolicy}}\right)^2 \frac{\loglogTerm(\round, \errorProb)}{\round}.
        \end{equation}
        Therefore, we can bound the cumulative Neyman regret during the clipping phase as 
        \begin{align}
            \sum_{\round = \clippingPhaseStoppingTime + 1}^{\numRounds} \neymanCost(\policy_\round) - \neymanCost(\neymanPolicy)
                &\leq 16\left( \frac{1}{\sqrtSecondMoment{0} + \sqrtSecondMoment{1}} \right)^2 \left( \frac{1}{\sqrt{\sqrtSecondMoment{0} \left( 1 - \neymanPolicy \right)}} + \frac{1}{\sqrt{\sqrtSecondMoment{1} \neymanPolicy}}\right)^2 \sum_{\round = \clippingPhaseStoppingTime + 1}^{\numRounds}  \frac{\loglogTerm(\round, \errorProb)}{\round} \\
                &\leq 16\left( \frac{1}{\sqrtSecondMoment{0} + \sqrtSecondMoment{1}} \right)^2 \left( \frac{1}{\sqrt{\sqrtSecondMoment{0} \left( 1 - \neymanPolicy \right)}} + \frac{1}{\sqrt{\sqrtSecondMoment{1} \neymanPolicy}}\right)^2 \sum_{\round = 1}^{\numRounds}  \frac{\loglogTerm(\round, \errorProb)}{\round} \\
                &\leq 16\left( \frac{1}{\sqrtSecondMoment{0} + \sqrtSecondMoment{1}} \right)^2 \left( \frac{1}{\sqrt{\sqrtSecondMoment{0} \left( 1 - \neymanPolicy \right)}} + \frac{1}{\sqrt{\sqrtSecondMoment{1} \neymanPolicy}}\right)^2 \loglogTerm(\numRounds, \errorProb) \log(\numRounds).
        \end{align}

        Combining these bounds we see that the Neyman regret can be bounded as
        \begin{equation}
            \clippingPhaseSimpleNeymanRegret \cdot \clippingPhaseStoppingTime + 16\left( \frac{1}{\sqrtSecondMoment{0} + \sqrtSecondMoment{1}} \right)^2 \left( \frac{1}{\sqrt{\sqrtSecondMoment{0} \left( 1 - \neymanPolicy \right)}} + \frac{1}{\sqrt{\sqrtSecondMoment{1} \neymanPolicy}}\right)^2 \loglogTerm(\numRounds, \errorProb) \log(\numRounds) = \tildeBigO(\log(\numRounds)),
         \end{equation}
        which gives the desired result.

    
    \neymanregretdecomposition*
    
    \begin{proof}
        The proof follows from the following series of algebraic manipulations:
        \begin{align}
                \neymanCost(\neymanPolicy + \policyDifference_\round) - \neymanCost(\neymanPolicy) 
                    &= \frac{\secondMoment{1}}{\neymanPolicy + \policyDifference_\round} + \frac{\secondMoment{0}}{(1 - \neymanPolicy - \policyDifference_\round)} - \frac{\secondMoment{1}}{\neymanPolicy} + \frac{\secondMoment{0}}{(1 - \neymanPolicy)}\\
                    &= \policyDifference_\round\left( \frac{\secondMoment{0}}{(1 - \neymanPolicy)(1 - \neymanPolicy - \policyDifference_\round)} - \frac{\secondMoment{1}}{\neymanPolicy(\neymanPolicy + \policyDifference_\round)} \right) \\
                    &\overset{(a)}{=} \policyDifference_\round\left( \frac{\secondMoment{0}}{\left( \frac{\sqrtSecondMoment{0}}{\sqrtSecondMoment{0} + \sqrtSecondMoment{1}} \right)\left( \frac{\sqrtSecondMoment{0}}{\sqrtSecondMoment{0} + \sqrtSecondMoment{1}} - \policyDifference_\round \right)} 
                            - \frac{\secondMoment{1}}{\left( \frac{\sqrtSecondMoment{1}}{\sqrtSecondMoment{0} + \sqrtSecondMoment{1}} \right)\left( \frac{\sqrtSecondMoment{1}}{\sqrtSecondMoment{0} + \sqrtSecondMoment{1}} + \policyDifference_\round \right)} \right) \\
                    &= \policyDifference_\round\left( \frac{\secondMoment{0}\left( \sqrtSecondMoment{0} + \sqrtSecondMoment{1} \right)^2}{\secondMoment{0}- \sqrtSecondMoment{0}\left( \sqrtSecondMoment{0} + \sqrtSecondMoment{1} \right) \policyDifference_\round} 
                        - \frac{\secondMoment{1}\left( \sqrtSecondMoment{0} + \sqrtSecondMoment{1} \right)^2}{\secondMoment{1}- \sqrtSecondMoment{1}\left( \sqrtSecondMoment{0} + \sqrtSecondMoment{1} \right) \policyDifference_\round}  \right) \\
                    &= \policyDifference_\round\left( \left[ \frac{\secondMoment{0}\left( \sqrtSecondMoment{0} + \sqrtSecondMoment{1} \right)^2}{\secondMoment{0}- \sqrtSecondMoment{0}\left( \sqrtSecondMoment{0} + \sqrtSecondMoment{1} \right) \policyDifference_\round} 
                        - \left( \sqrtSecondMoment{0} + \sqrtSecondMoment{1} \right)^{2} \right] \right. \nonumber \\
                    & \qquad \left. + \left[ \left( \sqrtSecondMoment{0} + \sqrtSecondMoment{1} \right)^{2} 
                        - \frac{\secondMoment{1}\left( \sqrtSecondMoment{0} + \sqrtSecondMoment{1} \right)^2}{\secondMoment{1}- \sqrtSecondMoment{1}\left( \sqrtSecondMoment{0} + \sqrtSecondMoment{1} \right) \policyDifference_\round} \right]  \right) \\
                    &= \policyDifference_\round^2\left( \sqrtSecondMoment{0} + \sqrtSecondMoment{1} \right)^3\left( \frac{1}{\sqrtSecondMoment{0} - \left( \sqrtSecondMoment{0} + \sqrtSecondMoment{1} \right) \policyDifference_\round} - \frac{1}{\sqrtSecondMoment{1} - \left( \sqrtSecondMoment{0} + \sqrtSecondMoment{1} \right) \policyDifference_\round}\right),
        \end{align}
        where in $(a)$ we have used the fact that $\neymanPolicy = \frac{\sqrtSecondMoment{1}}{\sqrtSecondMoment{0} + \sqrtSecondMoment{1}}$.    
    \end{proof}

    \subsection{Clipping Phase}
        We now cover various proofs related to the analysis of the clipping phase of our algorithm.

        We begin by proving Lemma~\ref{lem:clipping-phase-length-bound} which we restate for the reader's convenience.

        \clippingphaselengthbound*
        
        \begin{proof}
            To begin, we observe that since the function $x, y \mapsto \frac{x}{x + y}$ is monotonic increasing (resp. decreasing) in $x$ (resp. $y$) we have (on the event $\goodEvent$) that
            \begin{equation}~\label{eq:policy-ci}
                    \unclippedPolicy_{\round + 1} \in  \left[\frac{\sqrtSecondMoment{1} - \sqrtSecondMomentConfidenceWidth(\actionCount{\round}{1}, \sqrtSecondMomentErrorProb)}{\sqrtSecondMoment{0} + \sqrtSecondMomentConfidenceWidth(\actionCount{\round}{0}, \sqrtSecondMomentErrorProb) + \sqrtSecondMoment{1} - \sqrtSecondMomentConfidenceWidth(\actionCount{\round}{1}, \sqrtSecondMomentErrorProb)}, \frac{\sqrtSecondMoment{1} + \sqrtSecondMomentConfidenceWidth(\actionCount{\round}{1}, \sqrtSecondMomentErrorProb)}{\sqrtSecondMoment{0} - \sqrtSecondMomentConfidenceWidth(\actionCount{\round}{0}, \sqrtSecondMomentErrorProb) + \sqrtSecondMoment{1} + \sqrtSecondMomentConfidenceWidth(\actionCount{\round}{1}, \sqrtSecondMomentErrorProb)} \right].
            \end{equation}
            We note the above interval is random because $\actionCount{\round}{\action}$ is random.
            In order to construct bounds on $\actionCount{\round}{\action}$ we use the fact that $\policy_\round \in \left[ \clippingSequence_\round, 1 - \clippingSequence_\round \right]$ so that an integral-sum argument demonstrates
            \begin{equation}
                \sum_{\timeIndex=1}^{\round} \policy_{\timeIndex} \in \left[ \frac{1}{2}\cdot\frac{\round^{1 - \clippingExponent} - 1}{1 - \clippingExponent}, \frac{1}{2}\cdot\frac{\round^{1 - \clippingExponent}}{1 - \clippingExponent}\right].
            \end{equation}
            Therefore, on the event $\goodEvent$, we obtain
            \begin{align}
                \actionCount{\round}{1} \in \countInterval(\round, \countErrorProb) 
                    &\definedAs \left[ \frac{1}{2}\cdot \frac{\round^{1 - \clippingExponent} - 1}{1 - \clippingExponent} - \countConfidenceWidth(\round, \countErrorProb), 
                        \round - \frac{1}{2}\cdot\frac{\round^{1 - \clippingExponent}}{1 - \clippingExponent} + \countConfidenceWidth(\round, \countErrorProb) \right] \nonumber \\
                    &= \left[ \frac{1}{2}\cdot \frac{\round^{1 - \clippingExponent} - 1}{1 - \clippingExponent} - \sqrt{\round \cdot \loglogTerm(\round, \countErrorProb)}, 
                        \round - \frac{1}{2}\cdot\frac{\round^{1 - \clippingExponent}}{1 - \clippingExponent} + \sqrt{\round \cdot \loglogTerm(\round, \countErrorProb)}\right],
            \end{align}
            where we have set $\loglogTerm(\round, \errorProb) = \sqrt{.7225\left( \log \log \round + 0.72 \log \frac{5.2}{\errorProb}\right)}$.
            
            Our strategy moving forward will be to use these bounds on $\actionCount{\round}{1}$ to construct a time $\clippingPhaseStoppingTime$ such that for all $\round \geq \clippingPhaseStoppingTime$, we have $\unclippedPolicy_{\round + 1} \in \left[ \clippingSequence_{\round + 1}, 1 - \clippingSequence_{\round + 1}\right]$.
            We demonstrate how to do so in order to guarantee $\unclippedPolicy_{\round + 1} \geq \clippingSequence_{\round + 1}$ as the other case is entirely analogous.
            Observe that our initial (random) lower bound on $\unclippedPolicy_{\round + 1}$ together with our bounds on $\actionCount{\round}{1}$ imply that on the event $\goodEvent$, we have
            \begin{align} 
                \unclippedPolicy_{\round + 1} 
                    &\geq \min_{\countIndex \in \countInterval(\round, \countErrorProb)} 
                        \frac{\sqrtSecondMoment{1} - \sqrt{\frac{\loglogTerm(\highlight{\countIndex}, \sqrtSecondMomentErrorProb)}{\sqrtSecondMoment{1} \cdot \countIndex}}}
                        {\sqrtSecondMoment{0} + \sqrt{\frac{\loglogTerm(\highlight{\round - \countIndex}, \sqrtSecondMomentErrorProb)}{\sqrtSecondMoment{0} \cdot (\round - \countIndex)}} 
                            + \sqrtSecondMoment{1} - \sqrt{\frac{\loglogTerm(\highlight{\countIndex}, \sqrtSecondMomentErrorProb)}{\sqrtSecondMoment{1} \cdot \countIndex}}} \nonumber \\
                    &\geq \min_{\countIndex \in \countInterval(\round, \countErrorProb)} 
                        \frac{\sqrtSecondMoment{1} - \sqrt{\frac{\loglogTerm(\highlight{\round}, \sqrtSecondMomentErrorProb)}{\sqrtSecondMoment{1} \cdot \countIndex}}}
                        {\sqrtSecondMoment{0} + \sqrt{\frac{\loglogTerm(\highlight{\round}, \sqrtSecondMomentErrorProb)}{\sqrtSecondMoment{0} \cdot (\round - \countIndex)}} 
                            + \sqrtSecondMoment{1} -  \sqrt{\frac{\loglogTerm(\highlight{\round}, \sqrtSecondMomentErrorProb)}{\sqrtSecondMoment{1} \cdot \countIndex}}} \label{eq:count-interval-optimization-problem},
            \end{align}
            where the final inequality follows from the monotonic properties of the map $x,y \mapsto \frac{x}{x + y}$.
            Therefore, our objective is to upper bound the quantity 
            \begin{align}
                \clippingPhaseStoppingTimeLower &\definedAs \min \left\{ \round :  \min_{\countIndex \in \countInterval(\round, \countErrorProb)} 
                        \frac{\sqrtSecondMoment{1} - \sqrt{\frac{\loglogTerm(\round, \sqrtSecondMomentErrorProb)}{\sqrtSecondMoment{1} \cdot \countIndex}}}
                        {\sqrtSecondMoment{0} + \sqrt{\frac{\loglogTerm(\round, \sqrtSecondMomentErrorProb)}{\sqrtSecondMoment{0} \cdot (\round - \countIndex)}} 
                            + \sqrtSecondMoment{1} -  \sqrt{\frac{\loglogTerm(\round, \sqrtSecondMomentErrorProb)}{\sqrtSecondMoment{1} \cdot \countIndex}}} \geq \frac{1}{2} (\round + 1)^{-\clippingExponent} \right\} \\
                &\leq \min \left\{ \round :  \min_{\countIndex \in \countInterval(\round, \countErrorProb)} 
                        \frac{\sqrtSecondMoment{1} - \sqrt{\frac{\loglogTerm(\round, \sqrtSecondMomentErrorProb)}{\sqrtSecondMoment{1} \cdot \countIndex}}}
                        {\sqrtSecondMoment{0} + \sqrt{\frac{\loglogTerm(\round, \sqrtSecondMomentErrorProb)}{\sqrtSecondMoment{0} \cdot (\round - \countIndex)}} 
                            + \sqrtSecondMoment{1} -  \sqrt{\frac{\loglogTerm(\round, \sqrtSecondMomentErrorProb)}{\sqrtSecondMoment{1} \cdot \countIndex}}} \geq \frac{1}{2} \round^{-\clippingExponent} \right\} \label{eq:lower-policy-stopping-time},         
            \end{align}
            where the inequality follows from the fact that the LHS in increasing in $\round$ and the RHS is decreasing in $\round$.
            Letting $\optimalCountIndex$ denote the minimizer of equation~\eqref{eq:count-interval-optimization-problem}, by applying Lemma~\ref{lem:count-interval-optimization} we observe that
            \begin{equation}
                \optimalCountIndex \in \left\{ \frac{1}{2}\cdot \frac{\round^{1 - \clippingExponent} - 1}{1 - \clippingExponent} - \sqrt{\round \cdot \loglogTerm(\round, \countErrorProb)}, 
                        \round - \frac{1}{2}\cdot\frac{\round^{1 - \clippingExponent}}{1 - \clippingExponent} + \sqrt{\round \cdot \loglogTerm(\round, \countErrorProb)} \right\}.
            \end{equation}
            Therefore, we can compute an upper bound for each of the two cases so that taking the maximum of these bounds will result in an upper bound on equation~\eqref{eq:lower-policy-stopping-time}.



            We will demonstrate this for the case $\optimalCountIndex = \frac{1}{2}\cdot \frac{\round^{1 - \clippingExponent}}{1 - \clippingExponent} - \sqrt{\round \cdot \loglogTerm(\round, \countErrorProb)}$ since the other case is similar.
            After plugging this value of $\optimalCountIndex$ into equation~\eqref{eq:lower-policy-stopping-time}, rearranging terms shows that
            \begin{equation}\label{eq:lcs-1-step-1}
                \min \left\{ \round : \sqrtSecondMoment{1} \geq \frac{1}{2}\round^{-\clippingExponent}\left( \sqrtSecondMoment{0} + \sqrtSecondMoment{1} \right) + \frac{1}{2}\round^{-\frac{2 \clippingExponent + 1}{2}} \left( \frac{\loglogTerm\left( \round, \sqrtSecondMomentErrorProb \right)}{\sqrtSecondMoment{0} \fastIncreaseTerm\left( \round, \countErrorProb, \clippingExponent \right)} \right)^{\frac{1}{2}} + \round^{\frac{\clippingExponent - 1}{2}}\left( 1 - \round^{-\clippingExponent} \right) \left( \frac{\loglogTerm\left( \round, \sqrtSecondMomentErrorProb \right)}{\sqrtSecondMoment{1} \slowIncreaseTerm\left( \round, \countErrorProb, \clippingExponent \right)} \right)^{\frac{1}{2}}\right\},
            \end{equation}
            where 
            \begin{align*}
                \fastIncreaseTerm(\round, \errorProb, \clippingExponent) &= 1 + \round^{-\frac{1}{2}}\sqrt{\loglogTerm(\round, \errorProb)} + \frac{\round^{-1} - \round^{- \clippingExponent}}{2( 1 - \clippingExponent)} , \\
                \slowIncreaseTerm(\round, \errorProb, \clippingExponent) &= \frac{1 - \round^{\clippingExponent - 1}}{2(1 - \clippingExponent)} - \round^{\frac{2\clippingExponent - 1}{2}}\sqrt{\ell(\round, \errorProb)} .
            \end{align*}
            Defining $\worstCaseClippingExponent = \min \left\{ \clippingExponent, \frac{1 - \clippingExponent}{2} \right\}$, we can upper bound the RHS of equation~\eqref{eq:lcs-1-step-1} with
            \begin{equation}
                \round^{-\worstCaseClippingExponent} \left( \left( \sqrtSecondMoment{0} + \sqrtSecondMoment{1} \right) + \left( \frac{\loglogTerm\left( \round, \sqrtSecondMomentErrorProb \right)}{\sqrtSecondMoment{0} \fastIncreaseTerm\left( \round, \countErrorProb, \worstCaseClippingExponent \right)} \right)^{\frac{1}{2}} + \left( \frac{\loglogTerm\left( \round, \sqrtSecondMomentErrorProb \right)}{\sqrtSecondMoment{1} \slowIncreaseTerm\left( \round, \countErrorProb, \worstCaseClippingExponent \right)} \right)^{\frac{1}{2}} \right). 
            \end{equation}
            Rearranging terms demonstrates that it is sufficient to bound
            \begin{equation}\label{eq:lcs-1-step-2}
                \min \left\{ \round : \round^\worstCaseClippingExponent \geq  \frac{1}{\neymanPolicy} + \frac{\sqrt{\loglogTerm(\round, \sqrtSecondMomentErrorProb)}}{\sqrtSecondMoment{1}}\left[ \left( \frac{1}{\sqrtSecondMoment{0} \fastIncreaseTerm\left( \round, \countErrorProb, \worstCaseClippingExponent \right)} \right)^{\frac{1}{2}} + \left( \frac{1}{\sqrtSecondMoment{1} \slowIncreaseTerm\left( \round, \countErrorProb, \worstCaseClippingExponent \right)} \right)^{\frac{1}{2}} \right] \right\}.
            \end{equation}
            Squaring both sides and applying the inequality $(a + b)^{2} \leq a^{2} + b^{2}$ twice shows that we can bound
            \begin{equation}\label{eq:lcs-1-step-final}
                \min \left\{ \round : \round^{2\worstCaseClippingExponent} \geq  \frac{2}{\neymanPolicy^{2}} + \frac{4 \loglogTerm(\round, \sqrtSecondMomentErrorProb)}{\secondMoment{1}}\left[ \frac{1}{\sqrtSecondMoment{0} \fastIncreaseTerm\left( \round, \countErrorProb, \worstCaseClippingExponent \right)} + \frac{1}{\sqrtSecondMoment{1} \slowIncreaseTerm\left( \round, \countErrorProb, \worstCaseClippingExponent \right)} \right] \right\}.
            \end{equation}
            Next, we apply Lemma~\ref{lem:fast-increase-constant-lower-bound}~and~\ref{lem:slow-increase-constant-lower-bound} which show that when 
            \begin{equation*}
                \round \geq O\left( \max \left\{ \left( \frac{1}{1 - \worstCaseClippingExponent} \right)^{\frac{1}{\worstCaseClippingExponent}}, \left( \log\left( \frac{1}{\countErrorProb} \right) \right)^{\frac{1}{1 - 2\worstCaseClippingExponent}} \right\} \right), 
            \end{equation*}
            we have that $\slowIncreaseTerm(\round, \countErrorProb, \worstCaseClippingExponent), \fastIncreaseTerm(\round, \countErrorProb, \worstCaseClippingExponent) \geq \frac{1}{2}$.
            Applying Lemma~\ref{lem:polynomial-loglog-inversion} to equation~\eqref{eq:lcs-1-step-final} using the above bounds on $\slowIncreaseTerm, \fastIncreaseTerm$ demonstrates that 
            \begin{equation}
                \clippingPhaseStoppingTimeLower \leq \clippingPhaseStoppingTimeBaseLower^{\frac{1}{2\worstCaseClippingExponent}}
            \end{equation}
            where
            \begin{align}
                \clippingPhaseStoppingTimeBaseLower &\definedAs \neymanPolicyConstant{1} + \neymanPolicyConstant{2} \cdot \neymanPolicyConstant{3} \cdot \log \log \neymanPolicyConstant{1} \label{eq:clipping-phase-stopping-time-base-lower-definition}, \\
                \neymanPolicyConstant{1} &= \frac{2}{\neymanPolicy^2} + \frac{4}{\secondMoment{1}}\left( \frac{1}{\sqrtSecondMoment{0}} + \frac{1}{\sqrtSecondMoment{1}} \right) \log \left( \frac{5.2}{\sqrtSecondMomentErrorProb} \right), \\
                \neymanPolicyConstant{2} &= \frac{4}{\secondMoment{1}}\left( \frac{1}{\sqrtSecondMoment{0}} + \frac{1}{\sqrtSecondMoment{1}} \right), \\
                \neymanPolicyConstant{3} &= \frac{\log\log \neymanPolicyConstant{1}  - \log\left( 2\worstCaseClippingExponent \right)}{\log\log\neymanPolicyConstant{1}} \cdot \frac{\neymanPolicyConstant{1} \log  \neymanPolicyConstant{1}}{\neymanPolicyConstant{1} \log \neymanPolicyConstant{1} - \neymanPolicyConstant{2}}.
            \end{align}
            Repeating the argument for the other choice of $\optimalCountIndex$ yields the same result.
            
            Finally, we can repeat the above argument for the upper bound on $\unclippedPolicy_{\round + 1}$ which shows that $\clippingPhaseStoppingTimeUpper \leq \clippingPhaseStoppingTimeBaseUpper^{\frac{1}{2 \worstCaseClippingExponent}}$, where
            \begin{align}
                \clippingPhaseStoppingTimeBaseUpper &\definedAs \neymanPolicyConstant{1} + \neymanPolicyConstant{2} \cdot \neymanPolicyConstant{3} \cdot \log \log \neymanPolicyConstant{1} \label{eq:clipping-phase-stopping-time-base-upper-definition} \\
                \neymanPolicyConstant{1} &= \frac{2}{\left( 1 - \neymanPolicy \right)^2} + \frac{4}{\secondMoment{0}}\left( \frac{1}{\sqrtSecondMoment{0}} + \frac{1}{\sqrtSecondMoment{1}} \right) \log \left( \frac{5.2}{\sqrtSecondMomentErrorProb} \right)\\
                \neymanPolicyConstant{2} &= \frac{4}{\secondMoment{0}}\left( \frac{1}{\sqrtSecondMoment{0}} + \frac{1}{\sqrtSecondMoment{1}} \right).
            \end{align}
            
            Letting $\clippingPhaseStoppingTime = \max \left\{ \clippingPhaseStoppingTimeLower, \clippingPhaseStoppingTimeUpper \right\}$ gives the desired result.
            \end{proof}
            
        \subsection{Concentration Phase}
            In this section, we will prove Lemma~\ref{lem:concentration-phase-policy-bound} which we restate for the readers convenience below.
            \concentrationphasepolicybound* 
            \begin{proof}
                To begin, we fix $\round \geq \clippingPhaseStoppingTime$ and let $\timeIndex \in [\clippingPhaseStoppingTime, \round - 1]$.
                Invoking Lemma~\ref{lem:bounds-on-policy} implies that on the event $\goodEvent$ we have
                \begin{equation}
                    \actionCount{\timeIndex}{1} \in \left[ \policyMin \cdot \timeIndex - \sqrt{\timeIndex \loglogTerm(\timeIndex, \countErrorProb)}, \timeIndex - \policyMin \cdot \timeIndex + \sqrt{\timeIndex \loglogTerm(\timeIndex, \countErrorProb)} \right].
                \end{equation}
                We will use this to construct a lower bound on $\policy_{\timeIndex + 1}$ by solving the optimization problem in equation~\eqref{eq:count-interval-optimization-problem} using the interval defined above.
               Applying Lemma~\ref{lem:count-interval-optimization}, we can construct a lower bound by considering $\actionCount{\timeIndex}{1} \in \left\{ \policyMin \cdot \timeIndex - \sqrt{\timeIndex \loglogTerm(\timeIndex, \countErrorProb)}, \timeIndex - \policyMin \cdot \timeIndex - \sqrt{\timeIndex \loglogTerm(\timeIndex, \countErrorProb)} \right\}$.
               We demonstrate this for $\actionCount{\timeIndex}{1} = \policyMin \cdot \timeIndex - \sqrt{\timeIndex \loglogTerm(\timeIndex, \countErrorProb)}$.
               In this case, we have that
               \begin{align}
                   \policy_{\timeIndex + 1} 
                        &\geq \frac{\sqrtSecondMoment{1} - \sqrt{\frac{\loglogTerm(\timeIndex, \sqrtSecondMomentErrorProb)}{\sqrtSecondMoment{1} \actionCount{\timeIndex}{1}}}}{\sqrtSecondMoment{0} + \sqrt{\frac{\loglogTerm(\timeIndex, \sqrtSecondMomentErrorProb)}{\sqrtSecondMoment{1} (\timeIndex - \actionCount{\timeIndex}{1})}} + \sqrtSecondMoment{1} - \sqrt{\frac{\loglogTerm(\timeIndex, \sqrtSecondMomentErrorProb)}{\sqrtSecondMoment{1} \actionCount{\timeIndex}{1}}}} \\
                        &= \neymanPolicy \cdot \frac{\sqrtSecondMoment{0} + \sqrtSecondMoment{1}}{\sqrtSecondMoment{0} + \sqrtSecondMoment{1} + \combinedConcentrationConstantOne(\policyMin, \timeIndex)\sqrt{\frac{\loglogTerm(\timeIndex, \sqrtSecondMomentErrorProb)}{\timeIndex}}} - 
                            \frac{\concentrationConstantOne_1(\policyMin, \timeIndex) \sqrt{\frac{\loglogTerm(\timeIndex, \countErrorProb)}{\timeIndex}}}{\sqrtSecondMoment{0} + \sqrtSecondMoment{1} + \combinedConcentrationConstantOne(\policyMin, \timeIndex)\sqrt{\frac{\loglogTerm(\timeIndex, \sqrtSecondMomentErrorProb)}{\timeIndex}}} \\
                        &= \neymanPolicy \cdot \frac{\sqrt{\frac{\timeIndex}{\loglogTerm(\timeIndex, \sqrtSecondMomentErrorProb)}} \left( \sqrtSecondMoment{0} + \sqrtSecondMoment{1} \right)}{\sqrt{\frac{\timeIndex}{\loglogTerm(\timeIndex, \sqrtSecondMomentErrorProb)}} \left( \sqrtSecondMoment{0} + \sqrtSecondMoment{1} \right) + \combinedConcentrationConstantOne(\policyMin, \timeIndex)} - 
                            \frac{\concentrationConstantOne_1(\policyMin, \timeIndex)}{\sqrt{\frac{\timeIndex}{\loglogTerm(\timeIndex, \sqrtSecondMomentErrorProb)}} \left( \sqrtSecondMoment{0} + \sqrtSecondMoment{1} \right) + \combinedConcentrationConstantOne(\policyMin, \timeIndex)} \\
                        &\definedAs \underline{\policy}_{\timeIndex + 1},
               \end{align}
                where we have defined
               \begin{align}
                    \concentrationConstantOne_1(\policy, \timeIndex) &= \sqrt{\frac{1}{\sqrtSecondMoment{1}\left( \policy - \sqrt{\frac{\loglogTerm(\timeIndex, \countErrorProb)}{\timeIndex}} \right)}}, \\
                    \concentrationConstantOne_0(\policy, \timeIndex) &= \sqrt{\frac{1}{\sqrtSecondMoment{0}\left( (1 - \policy) - \sqrt{\frac{\loglogTerm(\timeIndex, \countErrorProb)}{\timeIndex}} \right)}}, \\
                    \combinedConcentrationConstantOne(\policy, \timeIndex) &= \concentrationConstantOne_0(\policy, \timeIndex) - \concentrationConstantOne_1(\policy, \timeIndex).
               \end{align}
                Using these bounds, on $\policy_{\timeIndex + 1}$ we observe that on the event $\goodEvent$ we have 
                \begin{align}
                    \actionCount{\round}{1} \geq \sum_{\timeIndex = 1}^{\round} \policy_\timeIndex - \sqrt{\round \loglogTerm(\round, \sqrtSecondMomentErrorProb)}
                        &= \sum_{\timeIndex = 1}^{\clippingPhaseStoppingTime} \policy_{\timeIndex} + \sum_{\timeIndex = \clippingPhaseStoppingTime + 1}^{\round} \policy_\timeIndex - \sqrt{\round \loglogTerm(\round, \sqrtSecondMomentErrorProb)}\\
                        &\geq \frac{\clippingPhaseStoppingTime^{1 - \clippingExponent} - 1}{2\left( 1 - \clippingExponent\right)} + \sum_{\timeIndex = \clippingPhaseStoppingTime + 1}^{\round} \underline{\policy}_{\timeIndex} - \sqrt{\round \loglogTerm(\round, \sqrtSecondMomentErrorProb)}
                \end{align}
                We bound $\sum_{\timeIndex = \clippingPhaseStoppingTime + 1}^{\round} \underline{\policy}_{\timeIndex}$ using Lemma~\ref{lem:refined-policy-tp1-lower-bound} so that

                \begin{equation}
                    \begin{aligned}
                        \actionCount{\round}{1} 
                        &\geq \neymanPolicy \cdot \round 
                        + \frac{\clippingPhaseStoppingTime^{1 - \clippingExponent} - 1}{2\left( 1 - \clippingExponent\right)} 
                        - \neymanPolicy\left( \clippingPhaseStoppingTime - 1 \right) \\
                        &\qquad - \sqrt{\round \loglogTerm(\round, \sqrtSecondMomentErrorProb)} 
                        \left( 
                            2 \frac{\combinedConcentrationConstantOne(\policyMin) + \concentrationConstantOne_1(\policyMin, \clippingPhaseStoppingTime)}
                            {\sqrtSecondMoment{0} + \sqrtSecondMoment{1}} + 1 
                        \right) \\
                        &\qquad - \frac{
                            \concentrationConstantOne_1(\policyMin, \clippingPhaseStoppingTime)
                        }{
                            \sqrt{\frac{\clippingPhaseStoppingTime}{\loglogTerm(\clippingPhaseStoppingTime, \sqrtSecondMomentErrorProb)}} 
                            \left( \sqrtSecondMoment{0} + \sqrtSecondMoment{1} \right) 
                            + \combinedConcentrationConstantOne(\policyMin, \clippingPhaseStoppingTime)
                        }\\
                        &\definedAs \neymanPolicy \cdot \round - \deviationConstant,
                    \end{aligned}
                \end{equation}
            where we have defined $\combinedConcentrationConstantOne(\policyMin) = \lim_{\round \rightarrow \infty} \combinedConcentrationConstantOne(\policyMin, \round)$.
            By plugging this value of $\actionCount{\round}{1}$ into equation~\eqref{eq:count-interval-optimization-problem}, we obtain
            \begin{equation}
                \policy_{\round + 1} \geq \underline{\policy}_{\round + 1} \definedAs \neymanPolicy \cdot \frac{\sqrt{\frac{\round}{\loglogTerm(\round, \sqrtSecondMomentErrorProb)}} \left( \sqrtSecondMoment{0} + \sqrtSecondMoment{1} \right)}{\sqrt{\frac{\round}{\loglogTerm(\round, \sqrtSecondMomentErrorProb)}} \left( \sqrtSecondMoment{0} + \sqrtSecondMoment{1} \right) + \combinedConcentrationConstantTwo(\neymanPolicy, \round)} - 
                            \frac{\concentrationConstantTwo_1(\neymanPolicy, \round)}{\sqrt{\frac{\round}{\loglogTerm(\round, \sqrtSecondMomentErrorProb)}} \left( \sqrtSecondMoment{0} + \sqrtSecondMoment{1} \right) + \combinedConcentrationConstantTwo(\neymanPolicy, \round)} \\
            \end{equation}
            where we have defined
            \begin{align}
                \concentrationConstantTwo_0(\policy, \round) = \sqrt{\frac{1}{\sqrtSecondMoment{0} \left( (1 - \policy) + \frac{\deviationConstant}{\round} \right)}} \\
                \concentrationConstantTwo_1(\policy, \round) = \sqrt{\frac{1}{\sqrtSecondMoment{1}\left( \policy - \frac{\deviationConstant}{\round} \right)}}, \\
                \combinedConcentrationConstantTwo(\policy, \round) = \concentrationConstantTwo_0(\policy, \round) - \concentrationConstantTwo_1(\policy, \round)
            \end{align}
            Therefore, we have
            \begin{align}
                \neymanPolicy - \policy_{\round + 1}
                    &\leq \neymanPolicy - \underline{\policy}_{\round + 1} \\
                    &= \sqrt{\frac{\loglogTerm(\round, \sqrtSecondMomentErrorProb)}{\round}}
                    \left( 
                        \frac{\neymanPolicy \concentrationConstantTwo_0 + (1 - \neymanPolicy)\concentrationConstantTwo_1}
                        {\sqrtSecondMoment{0} + \sqrtSecondMoment{1} + \sqrt{\frac{\loglogTerm(\round, \sqrtSecondMomentErrorProb)}{\round}}(\concentrationConstantTwo_0 - \concentrationConstantTwo_1)}
                    \right) \\
                    &= \sqrt{\frac{\loglogTerm(\round, \sqrtSecondMomentErrorProb)}{\round}}
                    \left( 
                        \frac{\neymanPolicy}{\sqrt{\sqrtSecondMoment{0}\left( \left( 1 - \neymanPolicy \right) + \frac{\deviationConstant}{\round}\right)}} 
                        + \frac{1 - \neymanPolicy}{\sqrt{\sqrtSecondMoment{1}\left( \neymanPolicy  - \frac{\deviationConstant}{\round}\right)}}
                    \right)
                    \left( 
                        \frac{1}{\sqrtSecondMoment{0} + \sqrtSecondMoment{1} + \sqrt{\frac{\loglogTerm(\round, \sqrtSecondMomentErrorProb)}{\round}}(\concentrationConstantTwo_0 - \concentrationConstantTwo_1)} 
                    \right) \\
                    &\leq 4 \sqrt{\frac{\loglogTerm(\round, \delta)}{\round}}
                    \left( 
                        \frac{1}{\sqrt{\sqrtSecondMoment{0}\left( 1 - \neymanPolicy \right)}} 
                        + \frac{1}{\sqrt{\sqrtSecondMoment{1} \neymanPolicy}} 
                    \right)
                    \left( 
                        \frac{1}{\sqrtSecondMoment{0} + \sqrtSecondMoment{1}} 
                    \right)
            \end{align}
            where the final inequality follows from the application of Lemmas~\ref{lem:deviation-constant-bound}~and~\ref{lem:second-deviation-term-bound} which shows that when $\round \geq \bigO(\clippingPhaseStoppingTime)$ we have that $\frac{\deviationConstant}{\round} \leq \frac{1}{2}\neymanPolicy$.
            \end{proof}

            \begin{lemma}~\label{lem:bounds-on-policy}
                Suppose we run \clipSMT with $\clippingSequence_\round = \frac{1}{2}\round^{-\clippingExponent}$ for some $\clippingExponent \in (0, 1)$ and let $\worstCaseClippingExponent = \min \left( \clippingExponent, \frac{1 - \clippingExponent}{2} \right)$.
                Then, on the event $\goodEvent$, for all $\round \geq 1$, we have that $\policy_\round \in [\policyMin, \policyMax] = \left[ \frac{1}{2}\clippingPhaseStoppingTimeBaseLower^{-\frac{\clippingExponent}{2\worstCaseClippingExponent}}, 1 - \frac{1}{2}\clippingPhaseStoppingTimeBaseUpper^{-\frac{\clippingExponent}{2\worstCaseClippingExponent}} \right]$ where $\clippingPhaseStoppingTimeBaseLower, \clippingPhaseStoppingTimeBaseUpper$ are respectively defined in equations~\eqref{eq:clipping-phase-stopping-time-base-lower-definition}~and~\eqref{eq:clipping-phase-stopping-time-base-upper-definition}.
            \end{lemma}
            \begin{proof}
                During the clipping phase, we know that $\policy_\round \in \left[ \clippingSequence_\round, 1 - \clippingSequence_\round \right]$.
                Additionally, once the clipping phase ends, we know that $\unclippedPolicy_\round = \policy_\round$ so that 
                \begin{equation}
                    \policy_{\round + 1} \in  \left[\frac{\sqrtSecondMoment{1} - \sqrtSecondMomentConfidenceWidth(\actionCount{\round}{1}, \sqrtSecondMomentErrorProb)}{\sqrtSecondMoment{0} + \sqrtSecondMomentConfidenceWidth(\actionCount{\round}{0}, \sqrtSecondMomentErrorProb) + \sqrtSecondMoment{1} - \sqrtSecondMomentConfidenceWidth(\actionCount{\round}{1}, \sqrtSecondMomentErrorProb)}, \frac{\sqrtSecondMoment{1} + \sqrtSecondMomentConfidenceWidth(\actionCount{\round}{1}, \sqrtSecondMomentErrorProb)}{\sqrtSecondMoment{0} - \sqrtSecondMomentConfidenceWidth(\actionCount{\round}{0}, \sqrtSecondMomentErrorProb) + \sqrtSecondMoment{1} + \sqrtSecondMomentConfidenceWidth(\actionCount{\round}{1}, \sqrtSecondMomentErrorProb)} \right].
                \end{equation}
                It is easy to see that the above bonunds are monotonic in $\round$---the lower bound is monotonically increasing and the upper bound is monotonically decreasing---which implies that $\policy_\round$ takes its minimum and maximum values at the end of the clipping phase.
                Therefore, we see that for all $\round \geq 1$, we have that
                \begin{equation}
                    1 - \frac{1}{2}\clippingPhaseStoppingTimeBaseUpper^{-\frac{\clippingExponent}{2 \worstCaseClippingExponent}} \geq 1 - \frac{1}{2}\clippingPhaseStoppingTime^{-\clippingExponent} = 1 - \clippingSequence_{\clippingPhaseStoppingTime} \geq  \policy_\round \geq \clippingSequence_{\clippingPhaseStoppingTime} = \frac{1}{2}\clippingPhaseStoppingTime^{-\clippingExponent} \geq \frac{1}{2}\clippingPhaseStoppingTimeBase^{-\frac{\clippingExponent}{2 \worstCaseClippingExponent}},
                \end{equation}
                where the first and last inequality follows from applying Lemma~\ref{lem:clipping-phase-length-bound} which shows that $\clippingPhaseStoppingTime \leq \clippingPhaseStoppingTimeBase^{\frac{1}{2\worstCaseClippingExponent}}$.
            \end{proof}

\section{Supporting Lemmas} \label{sec:supporting-lemma}
    \subsection{Intermediate Steps}
        \begin{lemma}\label{lem:refined-policy-tp1-lower-bound}
                Define
                \begin{equation}
                    \underline \policy_{\timeIndex + 1} = \neymanPolicy \cdot \frac{\sqrt{\frac{\timeIndex}{\loglogTerm(\timeIndex, \sqrtSecondMomentErrorProb)}} \left( \sqrtSecondMoment{0} + \sqrtSecondMoment{1} \right)}{\sqrt{\frac{\timeIndex}{\loglogTerm(\timeIndex, \sqrtSecondMomentErrorProb)}} \left( \sqrtSecondMoment{0} + \sqrtSecondMoment{1} \right) + \combinedConcentrationConstantOne(\policy_{\min}, \timeIndex)} - 
                    \frac{\concentrationConstantOne_1(\policy_{\min}, \timeIndex)}{\sqrt{\frac{\timeIndex}{\loglogTerm(\timeIndex, \sqrtSecondMomentErrorProb)}} \left( \sqrtSecondMoment{0} + \sqrtSecondMoment{1} \right) + \combinedConcentrationConstantOne(\policy_{\min}, \timeIndex)}.
                \end{equation}
                Then we have that 
                    \begin{equation}
                        \sum_{\timeIndex = \clippingPhaseStoppingTime + 1}^{\round} \underline{\policy}_{\timeIndex} \geq \neymanPolicy (\round - \clippingPhaseStoppingTime - 1) - 2\sqrt{\round \loglogTerm(\round, \sqrtSecondMomentErrorProb)}\left( \frac{\combinedConcentrationConstantOne(\policyMin)  + \concentrationConstantOne_1(\policyMin, \clippingPhaseStoppingTime)}{\sqrtSecondMoment{0} + \sqrtSecondMoment{1}} \right) - \frac{\combinedConcentrationConstantOne(\policy_{\min}, \clippingPhaseStoppingTime)}
                            {\sqrt{\frac{\clippingPhaseStoppingTime}{\loglogTerm(\clippingPhaseStoppingTime, \sqrtSecondMomentErrorProb)}} \left( \sqrtSecondMoment{0} + \sqrtSecondMoment{1} \right) + \combinedConcentrationConstantOne(\policy_{\min}, \clippingPhaseStoppingTime)},
                    \end{equation}
        \end{lemma}
        \begin{proof}
                    We begin by observing 
                    \begin{align}
                        \sum_{\timeIndex = \clippingPhaseStoppingTime + 1}^{\round} \underline{\policy}_{\timeIndex}
                            &= \sum_{\timeIndex = \clippingPhaseStoppingTime}^{\round - 1} \underline{\policy}_{\timeIndex + 1}\\
                            &= \sum_{\timeIndex = \clippingPhaseStoppingTime}^{\round - 1} \neymanPolicy \cdot \frac{\sqrt{\frac{\timeIndex}{\loglogTerm(\timeIndex, \sqrtSecondMomentErrorProb)}} \left( \sqrtSecondMoment{0} + \sqrtSecondMoment{1} \right)}{\sqrt{\frac{\timeIndex}{\loglogTerm(\timeIndex, \sqrtSecondMomentErrorProb)}} \left( \sqrtSecondMoment{0} + \sqrtSecondMoment{1} \right) + \combinedConcentrationConstantOne(\policy_{\min}, \timeIndex)} - 
                                \frac{\concentrationConstantOne_1(\policy_{\min}, \timeIndex)}{\sqrt{\frac{\timeIndex}{\loglogTerm(\timeIndex, \sqrtSecondMomentErrorProb)}} \left( \sqrtSecondMoment{0} + \sqrtSecondMoment{1} \right) + \combinedConcentrationConstantOne(\policy_{\min}, \timeIndex)} \\
                            &\geq \neymanPolicy \cdot \underbrace{\sum_{\timeIndex = \clippingPhaseStoppingTime}^{\round - 1}  \frac{\sqrt{\frac{\timeIndex}{\loglogTerm(\round, \sqrtSecondMomentErrorProb)}} \left( \sqrtSecondMoment{0} + \sqrtSecondMoment{1} \right)}{\sqrt{\frac{\timeIndex}{\loglogTerm(\round, \sqrtSecondMomentErrorProb)}} \left( \sqrtSecondMoment{0} + \sqrtSecondMoment{1} \right) + \combinedConcentrationConstantOne(\policy_{\min})}}_{\text{Term 1}} 
                            - \underbrace{\sum_{\timeIndex = \clippingPhaseStoppingTime}^{\round - 1} \frac{\concentrationConstantOne_1(\policy_{\min}, \clippingPhaseStoppingTime)}{\sqrt{\frac{\timeIndex}{\loglogTerm(\round, \sqrtSecondMomentErrorProb)}} \left( \sqrtSecondMoment{0} + \sqrtSecondMoment{1} \right) + \combinedConcentrationConstantOne(\policy_{\min}, \clippingPhaseStoppingTime)}}_{\text{Term 2}} ,
                    \end{align}
                    where we have set
                    \begin{equation*}
                        \combinedConcentrationConstantOne(\policy) = \sqrt{\frac{1}{\sqrtSecondMoment{0}\left(1 - \policy \right)}} - \sqrt{\frac{1}{\sqrtSecondMoment{1} \policy}}
                    \end{equation*}
                    and the inequality follows from the monotonic properties of the map $x, y \mapsto \frac{x}{x + y}$ combined with the fact that $\combinedConcentrationConstantOne(\policy, \round)$ is increasing in $\round$ and $\concentrationConstantOne_1(\policy, \round)$ is decreasing in $\round$.
                    From here, we lower bound $\text{Term 1}$ and upper bound $\text{Term 2}$.
                    
                    To lower bound $\text{Term 1}$, let $c_1 = \frac{\sqrtSecondMoment{0} + \sqrtSecondMoment{1}}{\sqrt{\loglogTerm(\round, \sqrtSecondMomentErrorProb)}}$ and $c_2 = h(\policy_{\min})$.
                    Then we have
                    \begin{align}
                        \text{Term 1} 
                            &= \sum_{\timeIndex = \clippingPhaseStoppingTime}^{\round - 1} \frac{\sqrt{\timeIndex} c_1}{\sqrt{\timeIndex} c_1 + c_2} \\
                            &= (\round - \clippingPhaseStoppingTime - 1) - \sum_{\timeIndex = \clippingPhaseStoppingTime}^{\round} \frac{c_2}{\sqrt{\timeIndex} c_1 + c_2} \\
                            &\geq (\round - \clippingPhaseStoppingTime - 1) - 2 \frac{c_2}{c_1}\sqrt{\round} \\
                            &= (\round - \clippingPhaseStoppingTime - 1) - 2 \frac{\combinedConcentrationConstantOne(\policy_{\min})}{\sqrtSecondMoment{0} + \sqrtSecondMoment{1}} \sqrt{\round \loglogTerm(\round, \sqrtSecondMomentErrorProb)}
                    \end{align}
                    where the inequality follows Lemma~\ref{lem:sqrt-integral-sum-upper}.

                    To upper bound $\text{Term 2}$ we similarly apply Lemma~\ref{lem:sqrt-integral-sum-upper} so that
                    \begin{equation}
                            \text{Term 2} \leq  \frac{\concentrationConstantOne_1(\policy_{\min}, \clippingPhaseStoppingTime)}
                            {\sqrt{\frac{\clippingPhaseStoppingTime}{\loglogTerm(\clippingPhaseStoppingTime, \sqrtSecondMomentErrorProb)}} \left( \sqrtSecondMoment{0} + \sqrtSecondMoment{1} \right) + \combinedConcentrationConstantOne(\policy_{\min}, \clippingPhaseStoppingTime)} + 2 \frac{ \concentrationConstantOne_1(\policy_{\min}, \clippingPhaseStoppingTime) \sqrt{\round \loglogTerm(\round, \sqrtSecondMomentErrorProb)}}
                            {\sqrtSecondMoment{0} + \sqrtSecondMoment{1}}
                    \end{equation}
                    Combining these bounds shows that
                    \begin{equation}
                        \sum_{\timeIndex = \clippingPhaseStoppingTime + 1}^{\round} \underline{\policy}_{\timeIndex} \geq \neymanPolicy (\round - \clippingPhaseStoppingTime - 1) - 2\sqrt{\round \loglogTerm(\round, \sqrtSecondMomentErrorProb)}\left( \frac{\combinedConcentrationConstantOne(\policyMin)  + \concentrationConstantOne_1(\policyMin, \clippingPhaseStoppingTime)}{\sqrtSecondMoment{0} + \sqrtSecondMoment{1}} \right) - \frac{\combinedConcentrationConstantOne(\policy_{\min}, \clippingPhaseStoppingTime)}
                            {\sqrt{\frac{\clippingPhaseStoppingTime}{\loglogTerm(\clippingPhaseStoppingTime, \sqrtSecondMomentErrorProb)}} \left( \sqrtSecondMoment{0} + \sqrtSecondMoment{1} \right) + \combinedConcentrationConstantOne(\policy_{\min}, \clippingPhaseStoppingTime)},
                    \end{equation}
                    thus proving the desired result.
                \end{proof}

        \begin{lemma}\label{lem:deviation-constant-bound}
            Define
            \begin{equation}
                \deviationConstant = \sqrt{\round \loglogTerm(\round, \errorProb)}
                    \left( 2\frac{ \combinedConcentrationConstantOne(\policyMin) + \concentrationConstantOne_1(\policyMin, \clippingPhaseStoppingTime)
                                }
                                {\sqrtSecondMoment{0} + \sqrtSecondMoment{1}
                                }  + 1 
                    \right) 
                    + \frac{
                                \concentrationConstantOne_1(\policyMin, \clippingPhaseStoppingTime)
                            }
                            {
                                \sqrt{\frac{\clippingPhaseStoppingTime}{\loglogTerm(\clippingPhaseStoppingTime, \sqrtSecondMomentErrorProb)}} 
                                \left( \sqrtSecondMoment{0} + \sqrtSecondMoment{1} \right) 
                                + \combinedConcentrationConstantOne(\policyMin, \clippingPhaseStoppingTime)
                            }
                    + \neymanPolicy\left( \clippingPhaseStoppingTime - 1 \right)
                    -\frac{\clippingPhaseStoppingTime^{1 - \clippingExponent} - 1}{2\left( 1 - \clippingExponent\right)}
            \end{equation}
            If $\round \geq 6 \clippingPhaseStoppingTime$, then $\frac{\deviationConstant}{\round} \leq \frac{1}{2}\neymanPolicy$.
        \end{lemma}

        \begin{proof}
            Note that it is sufficient to bound the first three terms since we are subtracting the fourth term.

            For the first term, we observe that when $\clippingPhaseStoppingTime^{1 - 2\clippingExponent} \geq 16 \loglogTerm(\clippingPhaseStoppingTime, \errorProb)$, which is satisfied by our definition of $\clippingPhaseStoppingTime$, we have that $\concentrationConstantOne_1(\policyMin, \clippingPhaseStoppingTime) \leq \sqrt{\frac{2}{\sqrtSecondMoment{1} \policyMin}}$.
            Therefore, some algebra shows that
            \begin{equation}
                2\frac{ \combinedConcentrationConstantOne(\policyMin) + \concentrationConstantOne_1(\policyMin, \clippingPhaseStoppingTime)
                                }
                                {\sqrtSecondMoment{0} + \sqrtSecondMoment{1}
                                }  + 1
                \leq 2 \clippingPhaseStoppingTime^{\frac{\clippingExponent}{2}} \left(\frac{1}{\sqrtSecondMoment{0}} + \frac{1}{\sqrtSecondMoment{1}}\right) \left( \frac{1}{\sqrt{\sqrtSecondMoment{0}}} + \frac{1}{\sqrt{\sqrtSecondMoment{1}}} \right) + 1 \definedAs c_1
            \end{equation}
            Therefore if we want bound this term by $b_1 \neymanPolicy$, we require $\round \geq \frac{c_1^2}{b_1^2} \loglogTerm(\round, \errorProb)$. 
            We can apply Lemma~\ref{lem:polynomial-loglog-inversion} to bound this.

            \noindent Next, some algebra shows that when $\clippingPhaseStoppingTime^{1 - \clippingExponent} \geq \frac{3}{\sqrtSecondMoment{1}\left(\sqrtSecondMoment{0} + \sqrtSecondMoment{1}\right)^2} \loglogTerm(\clippingPhaseStoppingTime, \errorProb)$, which is satisfied by our definition of $\clippingPhaseStoppingTime$, we have that
            \begin{equation}
                \frac{
                    \concentrationConstantOne_1(\policyMin, \clippingPhaseStoppingTime)
                }
                {
                    \sqrt{\frac{\clippingPhaseStoppingTime}{\loglogTerm(\clippingPhaseStoppingTime, \sqrtSecondMomentErrorProb)}} 
                    \left( \sqrtSecondMoment{0} + \sqrtSecondMoment{1} \right) 
                    + \combinedConcentrationConstantOne(\policyMin, \clippingPhaseStoppingTime)
                }
                \leq 2.     
            \end{equation}
            As such, if we want to bound this term by $b_2$, we require $\round \geq \frac{2}{b_2 \neymanPolicy}$

            \noindent Finally, to bound the third term by $b_3 \neymanPolicy$, we observe that we require $\round \geq \frac{\clippingPhaseStoppingTime - 1}{b_2}$.

            \noindent Setting $b_1 = b_2 = b_3 = \frac{1}{6}$ and $c_1 = \frac{1}{2}$, and using the above results, see that when
            \begin{equation}
                \round \geq \max\left\{144 \loglogTerm(\round, \errorProb), 6\left( \clippingPhaseStoppingTime - 1 \right) , \frac{12}{\neymanPolicy}\right\}
            \end{equation}
            we have that $\frac{\deviationConstant}{\round} \leq \frac{1}{2} \neymanPolicy$.
        \end{proof}

        \begin{lemma}\label{lem:second-deviation-term-bound}
            Suppose $\round \geq 6 \clippingPhaseStoppingTime$.
            Then we have that
            \begin{equation}
                \sqrt{\frac{\loglogTerm(\round, \errorProb)}{\round}}\left(\concentrationConstantTwo_1 - \concentrationConstantTwo_0\right) \leq \frac{1}{2} \left(\sqrtSecondMoment{0} + \sqrtSecondMoment{1}\right).
            \end{equation}
        \end{lemma}
        \begin{proof}
            To being, we see that it is sufficient to find compute an upper bound on the smallest $\round$ such that
            \begin{equation*}
                \round \geq \loglogTerm(\round, \errorProb)\frac{\concentrationConstantTwo_1^2}{\sqrtSecondMoment{0} + \sqrtSecondMoment{1}}.
            \end{equation*}
            Next, we apply Lemma~\ref{lem:deviation-constant-bound} which shows that when $\round \geq 6 \clippingPhaseStoppingTime$, we have that $\concentrationConstantTwo_1^2 \leq \frac{2}{\sqrtSecondMoment{1} \neymanPolicy}$.
            Plugging this in and applying Lemma~\ref{lem:polynomial-loglog-inversion} gives the desired result.
        \end{proof}
    \subsection{Useful Tools}
    \begin{lemma}
        We have that
    \begin{align}
        &\frac{
            \sqrtSecondMoment{1} - \sqrt{\frac{\loglogTerm(\round, \errorProb)}{\sqrtSecondMoment{1} \actionCount{\round}{1}}}
        }{
            \sqrtSecondMoment{0} 
            + \sqrt{\frac{\loglogTerm(\round, \errorProb)}{\sqrtSecondMoment{0} \left(\round -  \actionCount{\round}{1} \right)}} 
            + \sqrtSecondMoment{1} 
            - \sqrt{\frac{\loglogTerm(\round, \errorProb)}{\sqrtSecondMoment{1} \actionCount{\round}{1}}}
        } \\
        &= \neymanPolicy 
        \frac{
            \sqrt{\frac{\round}{\loglogTerm(\round, \errorProb)}} 
            \left( \sqrtSecondMoment{0} + \sqrtSecondMoment{1} \right)
        }{
            \sqrt{\frac{\round}{\loglogTerm(\round, \errorProb)}} 
            \left( \sqrtSecondMoment{0} + \sqrtSecondMoment{1} \right) 
            + \concentrationConstantOne_0(\round, \actionCount{\round}{1}) 
            - \concentrationConstantOne_1(\round, \actionCount{\round}{1})
        } \\
        &\quad 
        - \frac{
            \concentrationConstantOne_1(\round, \actionCount{\round}{1})
        }{
            \sqrt{\frac{\round}{\loglogTerm(\round, \errorProb)}} 
            \left( \sqrtSecondMoment{0} + \sqrtSecondMoment{1} \right) 
            + \concentrationConstantOne_0(\round, \actionCount{\round}{1}) 
            - \concentrationConstantOne_1(\round, \actionCount{\round}{1})
        }
    \end{align}
    where 
    \begin{align}
        \concentrationConstantOne_0(\round, \countIndex) &= \sqrt{\frac{1}{\sqrtSecondMoment{0}\left( 1 - \frac{\countIndex}{\round} \right)}}, \\
        \concentrationConstantOne_0(\round, \countIndex) &= \sqrt{\frac{1}{\sqrtSecondMoment{1} \cdot \frac{\countIndex}{\round}}}. 
    \end{align}
    \end{lemma}
    \subsection{Concentration Results}\label{sec:concentration}
            \begin{lemma}\label{lem:concentration-bernoulli-counts}
                Let $\randomVariable_1, \randomVariable_2, \ldots$ be a sequence of random variables such that $\randomVariable_\round \sim \Bernoulli(\policy_\round)$ where $\policy_\round$ is $\filtration_{\round - 1}$ measurable and define $\actionCount{\round}{} = \sum_{\timeIndex = 1}^{\round} \randomVariable_{\timeIndex}$.
                Then, with probability at least $1 - \errorProb$, the following holds for all $\round \in \bbN$
                \begin{equation}
                    \left \lvert \actionCount{\round}{} - \sum_{\timeIndex = 1}^{\round} \policy_\round \right\rvert \leq \countConfidenceWidth(\round, \errorProb),
                \end{equation}
                where 
                \begin{equation}\label{eq:count-confidence-width}
                    \countConfidenceWidth(\round, \errorProb) = 0.85 \sqrt{\round \left( \log \log \round + 0.72 \log \left( \frac{5.2}{\errorProb} \right) \right)}.
                \end{equation}
            \end{lemma}
            \begin{proof}
                Define $M_\round^\lambda = \exp\left( \lambda(\randomVariable - p_\round) - \frac{\lambda^2}{8} \right)$.
                Note that by definition, $\randomVariable_\round \in [0, 1]$ almost surely with $\bbE\left[ \randomVariable_\round \mid \filtration_{\round - 1} \right] = p_\round$ which implies that the following holds for every $\lambda \in \mathbb R$
                \begin{equation}
                    \bbE\left[ M_\round^\lambda \mid \filtration_{\round - 1}\right] \leq 1.
                \end{equation}
                Therefore, $D_\round^\lambda = \prod_{\timeIndex = 1}^{\round} M_{\timeIndex}^\lambda$ is a test supermartingale and we can apply Theorem~1 from \cite{Howard2018TimeuniformNN} (see equation~(11)) to obtain the desired result.
            \end{proof} 
        
            \begin{lemma}\label{lem:sqrt-second-moment-concentration}
                Let $\randomVariable_1, \randomVariable_2, \ldots$ be a sequence of random variables such that $\randomVariable_\round \in [0, 1]$, $\mu = \bbE\left[ \randomVariable_\round \mid \filtration_{\round - 1} \right]$, and $ \secondMoment{} = \bbE\left[ X_\round^2 \mid \filtration_{\round - 1} \right]$.
                Define the empirical second moment as $\empiricalSecondMoment{\round}{} = \frac{1}{\round}\sum_{\timeIndex = 1}^{\round} \randomVariable_\timeIndex^{2}$.
                Then, with probability at least $1 - \errorProb$, the following holds for all $\round \in \bbN$
                \begin{equation}
                    \left\lvert \empiricalSqrtSecondMoment{\round}{} - \sqrtSecondMoment{} \right\rvert \leq \sqrtSecondMomentConfidenceWidth(\round, \errorProb)
                \end{equation}
                where
                \begin{equation}\label{eq:second-moment-confidence-width}
                    \sqrtSecondMomentConfidenceWidth(\round, \errorProb) = 0.85 \sqrt{\frac{ \left( \log \log \round + 0.72 \log \left( \frac{5.2}{\errorProb} \right) \right)}{\secondMoment{} \cdot \round}}.
                \end{equation}
            \end{lemma}
            \begin{proof}
            To see this, we first observe that
            \begin{equation*}
                \left\lvert \empiricalSqrtSecondMoment{\round}{} - \sqrtSecondMoment{} \right\rvert = \frac{\left\lvert \empiricalSecondMoment{\round}{} - \secondMoment{} \right\rvert}{\left\lvert \empiricalSqrtSecondMoment{\round}{} + \sqrtSecondMoment{} \right\rvert} \leq \frac{\left\lvert \empiricalSecondMoment{\round}{} - \secondMoment{} \right\rvert}{\left\lvert \sqrt{\secondMoment{}} \right\rvert}.
            \end{equation*}
            The result then follows by bounding $\left\lvert \empiricalSecondMoment{\round}{} - \secondMoment{} \right\rvert$ by applying Theorem~1 from \cite{Howard2018TimeuniformNN} (see equation~(11)).
            \end{proof}

            \begin{remark}\label{rem:second-moment-clt-argument}
                Note that in our above result, the width of the confidence sequences scale like $O \left( \frac{1}{\sqrt{\secondMoment{}}\cdot \round} \right)$.
                An application of the CLT along with the Delta Method shows that, asymptotically, the scaling with respect to $\frac{1}{\sqrt{\secondMoment{}}}$ is unavoidable.
            \end{remark}
    
    \subsection{Technical Results}
            \begin{lemma}\label{lem:count-interval-optimization}
                Let $\round, \alpha_0, \alpha_1, \gamma_0, \gamma_1 > 0$ be fixed, and define the function $f:(0, \round) \rightarrow \bbR$ by
                \begin{equation}
                    f(x) = \frac{\alpha_1 - \frac{\gamma_1}{\sqrt{x}}}{\alpha_0 + \frac{\gamma_0}{\sqrt{\round - x}} + \alpha_1 - \frac{\gamma_1}{\sqrt{x}}}.
                \end{equation}
                Given an interval $[s, r] \subseteq [1, \round]$, any solution $x^\star$ to the optimization problem
                    \begin{equation}
                    \min_{x \in [s, r]} f(x),
                    \end{equation}
                must satisfy $x^\star \in \{s, r\}$.
            \end{lemma}
            \begin{proof}
                Our proof will proceed by demonstrating that one of the preconditions of Lemma~\ref{lem:single-root-unimodal} is satisfied, from which the desired result naturally follows.
                To begin, we let $f'(x) = \frac{d}{dx} f(x)$ denote the derivative of $f(x)$.
                We compute $f'(x)$ and perform some simplifications to show that
                \begin{align}
                    f'(x) 
                    &= -\left(\frac{\left(\frac{\gamma_0}{2 (t - x)^{3/2}} + \frac{\gamma_1}{2 x^{3/2}}\right) (\alpha_1 - \frac{\gamma_1}{\sqrt{x}})}{(\alpha_0 + \alpha_1 + \frac{\gamma_0}{\sqrt{t - x}} - \frac{\gamma_1}{\sqrt{x}})^2}\right) + \frac{\gamma_1}{2 (\alpha_0 + \alpha_1 + \frac{\gamma_0}{\sqrt{t - x}} - \frac{\gamma_1}{\sqrt{x}}) x^{3/2}} \nonumber \\
                    &= \frac{(\gamma_0 \gamma_1 t + \alpha_0 \gamma_1 t \sqrt{t - x} - \alpha_0 \gamma_1 \sqrt{t - x} x - \alpha_1 \gamma_0 x^{3/2})}{2 (-\gamma_1 \sqrt{t - x} + \gamma_0 \sqrt{x} + \alpha_0 \sqrt{t - x} \sqrt{x} + \alpha_1 \sqrt{t - x} \sqrt{x})^2 \sqrt{t - x} \sqrt{x}} \label{eq:numerator-and-denominator}.
                \end{align}
                Observe that the denominator in~\eqref{eq:numerator-and-denominator} is always greater than zero.
                Therefore, $\sign(f'(x))$ is determined by the numerator which we will now show to be strictly decreasing.
                The derivative of the numerator in~\eqref{eq:numerator-and-denominator} is
                \begin{equation*}
                    -\left(\frac{3 (\alpha_0 \gamma_1 t + \alpha_1 \gamma_0 \sqrt{t - x} \sqrt{x} - \alpha_0 \gamma_1 x)}{2 \sqrt{t - x}}\right).
                \end{equation*}
                From here, we have that by assumption $\alpha_0, \alpha_1, \gamma_0, \gamma_1>0$ and $x < \round$ imply that the above quantity is strictly negative.
                Since the derivative of the numerator is strictly negative, we know that the numerator is strictly decreasing.
                Therefore, our earlier observation, in conjunction with this fact implies that one of the preconditions of Lemma~\ref{lem:single-root-unimodal} must hold, thus enabling its application, which in turn implies the desired result.
            \end{proof}

            The next lemma essentially shows that the minimum of a concave-unimodal function over a closed interval must occur at one of the boundaries of the interval.
            \begin{lemma}\label{lem:single-root-unimodal}
                Let $f:\mathcal D \rightarrow \mathbb R$ be any differential function such that its derivative, $f'$, satisfies one of the following conditions:
                \begin{enumerate}
                    \item $f'(x)> 0$  for all $x\in\mathcal D$
                    \item $f'(x) < 0$  for all $x\in\mathcal D$
                    \item There exists $c$ such that for all $x < c$, $f'(x) > 0$ and for all $x > c$, $f'(x) < c$.
                \end{enumerate}
                Then for any $[a, b] \subset \mathcal D$, any solution $x^\star$ to optimization problem,
                \begin{equation}
                    \min_{x \in [a, b]} f(x),
                \end{equation}
                must satisfy $x^\star \in \{a, b\}$.
            \end{lemma}
            \begin{proof}
                If $f'(x) > 0$ for all $x \in \mathcal D$, the function is monotonically increasing and the minimum will occur at $ x^\star = a$.
                If $f'(x) < 0$ for all $x \in \mathcal D$, the function is monotonically decreasing and the minimum will occur at $x^\star = b$.
                For the final case, let $c$ be as defined in the condition and let $\tilde x$ denote the minimum of $f$.
                If $a < \tilde x < c$ then $f(\tilde x) - f(a) = \int_{a}^{\tilde x} f'(t) dt > 0$ which is a contradiction.
                Similarly if $b > \tilde x > c$, then $f(b) - f(\tilde x) = \int_{\tilde x}^{b} f'(c) dc < 0$ which is also a contradiction.
                Therefore, for each of the cases, $x^\star$ must satisfy $x^\star \in \{a, b\}.$
            \end{proof}

            \begin{lemma}\label{lem:polynomial-loglog-inversion}
                Let $c_1, c_2, p > 0$ such that $\log c_1 > p$ and $c_1 \log c_1 > c_2$ and define
                \begin{equation}
                    \clippingPhaseStoppingTime = \min \left\{ \round : \round^{p} \geq c_1 + c_2 \log\log(\round)\right\}.
                \end{equation}
                We have that
                \begin{equation}
                    \clippingPhaseStoppingTime \leq \left( c_1 + c_2 \log\left(\log c_1 \right) \frac{\log\log c_1 - \log(p)}{\log\log c_1} \cdot \frac{c_1 \log c_1}{c_1 \log c_1 - c_2} \right)^{\frac{1}{p}}
                \end{equation}
            \end{lemma}
            \begin{proof}
                To prove this, we set
                \begin{equation*}
                    \round = \left( c_1 + ac_2 \log \log c_1 \right)^{\frac{1}{p}},
                \end{equation*}
                for some $a$ to be chosen later.
                Our objective is to show that
                \begin{equation*}
                    \log \log \left[ \left( c_1 + a c_2 \log \log c_1 \right)^\frac{1}{p} \right] \leq a \log \log c_1.
                \end{equation*}
                To do so, we observe that
                \begin{align*}
                    &\log \left( \log \left( \left( c_1 + a c_2 \log \log c_1 \right)^\frac{1}{p} \right) \right) \\
                    &= \log \left( \frac{1}{p}\log \left( c_1 + a c_2 \log \log c_1 \right) \right) \\
                    &= \log \left( \frac{1}{p}\left( \log(c_1) + \log\left( 1 + \frac{ac_2}{c_1} \log\log c_1 \right) \right)  \right) \\
                    &\leq \log \left( \frac{1}{p}\left( \log(c_1) + \frac{ac_2}{c_1} \log\log c_1 \right)  \right) \\
                    &= \log \left( \frac{1}{p} \log c_1 \right) + \log \left( 1 + \frac{a c_2}{c_1 \log c_1} \log \log c_1 \right) \\
                    &\leq \log \left( \frac{1}{p} \log c_1 \right) +  \frac{a c_2}{c_1 \log c_1} \log \log c_1, \\
                \end{align*}
                where the inequalities follow from applying the inequality $\log(1 + x) \leq x$.
                From here, we set $a$ so the final line above equals $a \log \log c_1$.
                In particular, by setting 
                \begin{equation*}
                    a = \frac{\log\log c_1 - \log(p)}{\log\log c_1} \cdot \frac{c_1 \log c_1}{c_1 \log c_1 - c_2},
                \end{equation*}
                the above series of inequalities proves that 
                \begin{equation*}
                    \log \log \left[ \left( c_1 + a c_2 \log \log c_1 \right)^\frac{1}{p} \right] \leq a \log \log c_1,
                \end{equation*}
                as desired.
            \end{proof}

            \begin{lemma}~\label{lem:fast-increase-constant-lower-bound}
                Fix $\clippingExponent, \errorProb \in (0, 1)$ and consider the function
                \begin{equation*}
                    \fastIncreaseTerm(\round, \errorProb, \clippingExponent) = 1 + \round^{-\frac{1}{2}}\sqrt{\ell(t, \errorProb)} +  \frac{\round^{-1} - \round^{- \clippingExponent}}{ 1 - \clippingExponent}.
                \end{equation*}
                For all $\round \geq \left( \frac{2}{1 - \clippingExponent} \right)^{\frac{1}{\clippingExponent}}$, we have that $g(\round, \errorProb, \clippingExponent) \geq \frac{1}{2}$.
            \end{lemma}
            \begin{proof}
                First note that
                \begin{equation*}
                    1 + \round^{-\frac{1}{2}}\sqrt{\ell(t, \errorProb)} + \frac{\round^{-1} - \round^{- \clippingExponent}}{ 1 - \clippingExponent} \geq 1 - \frac{\round^{- \clippingExponent}}{ 1 - \clippingExponent}. 
                \end{equation*}
                Solving the inequality
                \begin{equation*}
                    1 - \frac{\round^{- \clippingExponent}}{ 1 - \clippingExponent} \geq \frac{1}{2},
                \end{equation*}
                for $\round$ gives the desired result.
            \end{proof}
            \begin{corollary}\label{corr:fast-increase-constant-lower-bound-1/3}
                 For $\errorProb \in (0, 1)$, $\round \geq 27$ implies that $\fastIncreaseTerm(\round, \errorProb, \frac{1}{3}) \geq \frac{1}{2}$.
            \end{corollary}

            \begin{lemma}~\label{lem:slow-increase-constant-lower-bound}
                Fix $\clippingExponent \in \left( 0, \frac{1}{2} \right)$, $\errorProb \in \left( 0, 1 \right)$, and let
                \begin{equation*}
                    \slowIncreaseTerm(t, \errorProb, \clippingExponent) = \frac{1 - \round^{\clippingExponent - 1}}{2 \left( 1 - \clippingExponent \right)} - \round^{\frac{2\clippingExponent - 1}{2}}\sqrt{\ell(\round, \errorProb)}.
                \end{equation*}
                We have that  $\slowIncreaseTerm(\round, \errorProb, \clippingExponent) \geq \frac{1}{2}$ whenever
                \begin{equation*}
                    \round \geq \left( c_1 + c_2 \log\left(\log c_1 \right) \frac{\log\log c_1 - \log(1 - 2 \clippingExponent)}{\log\log c_1} \cdot \frac{c_1 \log c_1}{c_1 \log c_1 - c_2} \right)^{\frac{1}{1 - 2 \clippingExponent}},
                \end{equation*}
                where $c_1 = \frac{2}{\clippingExponent^{2}} + \frac{8\left( 1 - \clippingExponent\right)^{2}}{\clippingExponent^{2}} \log\left( \frac{5.2}{\delta} \right)$ and $c_2 = \frac{8\left( 1 - \clippingExponent\right)^{2}}{\clippingExponent^{2}}$.
            \end{lemma}

            \begin{proof}
                To begin, observe that 
                \begin{equation}
                    \frac{1 - \round^{\clippingExponent - 1}}{2 \left( 1 - \clippingExponent \right)} - \round^{\frac{2\clippingExponent - 1}{2}}\sqrt{\ell(\round, \errorProb)} \geq  \frac{1}{2\left( 1 - \clippingExponent \right)} - \round^{\frac{2\clippingExponent - 1}{2}}\sqrt{\ell(\round, \errorProb)} - \frac{\round^{\frac{2\clippingExponent - 1}{2}}}{2\left( 1 - \clippingExponent \right)},
                \end{equation}
                therefore it is sufficient to bound the quantity
                \begin{equation}
                    \min \left\{ \round : \frac{1}{2 \left( 1 - \clippingExponent \right)} -  \round^{\frac{2\clippingExponent - 1}{2}}\left( \sqrt{\ell(\round, \errorProb)} +  \frac{1}{2\left( 1 - \clippingExponent \right)} \right)   \geq \frac{1}{2} \right\}.
                \end{equation}
                Rearranging, we see that this is equivalent to bounding the quantity
                \begin{equation}
                    \min \left\{ \round : \round^{\frac{1}{2}-\clippingExponent} \geq \frac{2\left( 1 - \clippingExponent \right)}{\clippingExponent} \sqrt{\loglogTerm(\round, \errorProb)} + \frac{1}{\clippingExponent}\right\}.
                \end{equation}
                By squaring both sides and applying the inequality $(a + b)^2 \leq 2a^2 + 2b^2$ we see that it is sufficient to bound
                \begin{equation}
                    \min \left\{ \round : \round^{1 - 2 \clippingExponent} \geq \frac{2}{\clippingExponent^{2}} + \frac{8\left( 1 - \clippingExponent\right)^{2}}{\clippingExponent^{2}} \loglogTerm(\round, \errorProb) \right\}.
                \end{equation}
                Setting $c_1 = \frac{2}{\clippingExponent^{2}} + \frac{8\left( 1 - \clippingExponent\right)^{2}}{\clippingExponent^{2}} \log\left( \frac{5.2}{\delta} \right)$ and $c_2 = \frac{8\left( 1 - \clippingExponent\right)^{2}}{\clippingExponent^{2}}$ we can apply Lemma~\ref{lem:polynomial-loglog-inversion} to see that whenever
                \begin{equation*}
                    \round \geq \left( c_1 + c_2 \log\left(\log c_1 \right) \frac{\log\log c_1 - \log(1 - 2 \clippingExponent)}{\log\log c_1} \cdot \frac{c_1 \log c_1}{c_1 \log c_1 - c_2} \right)^{\frac{1}{1 - 2 \clippingExponent}},
                \end{equation*}
                we have that $\slowIncreaseTerm(\round, \errorProb, \clippingExponent) \geq \frac{1}{2}$, as desired.
            \end{proof}

            \begin{corollary}\label{corr:slow-increase-constant-lower-bound-1/3}
                For $\errorProb \in (0, 1)$, $\round \geq O\left( \log \left( \frac{1}{\delta} \right)^{3} \right)$ implies that $\slowIncreaseTerm\left( \round, \errorProb, \frac{1}{3} \right) \geq \frac{1}{2}.$
            \end{corollary}

            \begin{lemma}\label{lem:sqrt-integral-sum-upper}
                Fix $c_1, c_2, c_3, \clippingPhaseStoppingTime, \round$ such that $c_1, c_2 > 0$, $\clippingPhaseStoppingTime < \round$, and $c_2 \sqrt{\clippingPhaseStoppingTime} + c_3 > 0$.
                Then, we have that
                \begin{equation}
                    \sum_{\timeIndex = \clippingPhaseStoppingTime}^{\round} \frac{c_1}{c_2 \sqrt{\timeIndex} + c_3} \leq \frac{c_1}{c_2 \sqrt{\clippingPhaseStoppingTime} + c_3} + \frac{2 c_1}{c_2}\left( \sqrt{\round} - \sqrt{\clippingPhaseStoppingTime} \right)
                \end{equation}
            \end{lemma}
            \begin{proof}
                Observe that under the stated conditions, we have that $\frac{c_1}{c_2 \sqrt{\timeIndex} + c_3}$ is monotonically decreasing in $\timeIndex$.
                Therefore we can bound
                \begin{align*}
                    \sum_{\timeIndex = \clippingPhaseStoppingTime}^{\round} \frac{c_1}{c_2 \sqrt{\timeIndex} + c_3} 
                        &\leq \frac{c_1}{c_2 \sqrt{\clippingPhaseStoppingTime} + c_3} + \int_{\timeIndex = \clippingPhaseStoppingTime}^{\round} \frac{c_1}{c_2 \sqrt{\timeIndex} + c_3} ds \\
                        &\leq \frac{c_1}{c_2 \sqrt{\clippingPhaseStoppingTime} + c_3} + \left. \left( \frac{2 c_1}{c_2}\sqrt{\timeIndex} -  \frac{2 c_{1} c_{3}}{c_2^2} \log\left( c_3 + c_2 \sqrt{\timeIndex} \right)\right) \right\rvert_{\timeIndex = \clippingPhaseStoppingTime}^{\round} \\
                        &= \frac{c_1}{c_2 \sqrt{\clippingPhaseStoppingTime} + c_3} 
                            + \left( \frac{2 c_1}{c_2}\sqrt{\round} - \frac{2 c_{1} c_{3}}{c_2^2} \log\left( c_3 + c_2 \sqrt{\round} \right) \right) 
                            - \left( \frac{2 c_1}{c_2}\sqrt{\clippingPhaseStoppingTime} - \frac{2 c_{1} c_{3}}{c_2^2} \log\left( c_3 + c_2 \sqrt{\clippingPhaseStoppingTime} \right) \right)  \\
                        &\leq \frac{c_1}{c_2 \sqrt{\clippingPhaseStoppingTime} + c_3} + \frac{2 c_1}{c_2}\left( \sqrt{\round} - \sqrt{\clippingPhaseStoppingTime} \right)
                \end{align*}
            \end{proof}

\section{Discussion on Clipping Sequences}\label{app:clipping-sequences}

Recall that our proposed ClipSMT algorithm utilizes clipping sequence with polynomial decay so that $\clippingSequence_\round = \frac{1}{2} \round^{-\alpha}$ for $\alpha \in (0, 1)$.
It is natural to wonder if there are other valid choices for the  clipping sequence.
While there are, the choices of clipping sequences that will work depend on the assumptions that we make.

On one hand, if we do not assume a lower bound on $\secondMoment{\action}$, then we must require that $\sum_{\round} \clippingSequence_\round$ diverges as $\round \rightarrow \infty$.
To see why, suppose the sum converges, i.e $\lim_{\numRounds \rightarrow \infty} \sum_{\round = 1}^{\numRounds} \clippingSequence_\round = c$. 
Then, if we choose $\secondMoment{1}, \secondMoment{2}$ so that the length of the clipping phase is larger than $c$, this will ensure that $\policy_\round$ never converges to $\neymanPolicy$.
As a concrete example, this implies that in this most general setting, we should not use clipping sequence with exponential decay.
However, if we are willing to assume a lower bound on $\secondMoment{0}, \secondMoment{1}$, then we can use a similar argument in order to select the rate of decay for a clipping sequence whose sum converges.

\end{document}